\documentclass[runningheads]{llncs}
\usepackage[T1]{fontenc}
\usepackage{graphicx}
\usepackage{subcaption}
\usepackage{xparse}
\usepackage{tikzit}

\tikzstyle{none}=[inner sep=0pt]

\tikzstyle{label}=[fill=none, draw=none, inner sep=0pt, shape=circle, scale=0.6, tikzit fill={rgb,255: red,191; green,191; blue,191}, tikzit draw={rgb,255: red,191; green,191; blue,191}]

\tikzstyle{black}=[fill=black, draw=black, shape=circle, inner sep=0pt, minimum size=3pt, tikzit category=dots]

\tikzstyle{box}=[fill=white, draw=black, chamfered rectangle, chamfered rectangle corners={north east}, chamfered rectangle sep=2pt, inner sep=0pt, minimum height=1.4em, minimum width=1.4em, scale=0.8, tikzit category=boxes]
\tikzstyle{boxOp}=[fill=white, draw=black, chamfered rectangle, chamfered rectangle corners={north west}, chamfered rectangle sep=2pt, inner sep=0pt, minimum height=1.4em, minimum width=1.4em, scale=0.8, tikzit category=boxes]

\tikzstyle{map}=[fill=white, draw=black, shape=rounded rectangle, rounded rectangle left arc=none, minimum height=1.5em, minimum width=1.5em, scale=0.8, tikzit category=boxes]
\tikzstyle{mapOp}=[fill=white, draw=black, shape=rounded rectangle, rounded rectangle right arc=none, minimum height=1.5em, minimum width=1.5em, scale=0.8, tikzit category=boxes]

\tikzstyle{square}=[fill=white, draw=black, rectangle, inner sep=0pt, minimum height=1.4em, minimum width=1.4em, scale=0.8, tikzit category=boxes]

\tikzstyle{wire}=[-, tikzit category=wires]

\tikzset{baseline=-0.5ex,>=stealth'}
\tikzset{every picture/.append style={scale=0.5}}

\NewDocumentCommand{\boxCirc}{moo}{
    \begin{tikzpicture}
        \begin{pgfonlayer}{nodelayer}
            \node [style=box] (3) at (0, 0) {$#1$};
            \node [style=none] (6) at (-1, 0) {};
            \node [style=none] (7) at (1, 0) {};
            \IfNoValueF{#3}{
                \node [style=label] (4) at (-1.25, 0) {$#2$};
                \node [style=label] (5) at (1.25, 0) {$#3$};
            }
        \end{pgfonlayer}
        \begin{pgfonlayer}{edgelayer}
            \draw (6.center) to (3);
            \draw (7.center) to (3);
        \end{pgfonlayer}
    \end{tikzpicture}        
}

\NewDocumentCommand{\boxOpCirc}{moo}{
    \begin{tikzpicture}
        \begin{pgfonlayer}{nodelayer}
            \node [style=boxOp] (3) at (0, 0) {$#1$};
            \node [style=none] (6) at (-1, 0) {};
            \node [style=none] (7) at (1, 0) {};
            \IfNoValueF{#3}{
                \node [style=label] (4) at (-1.25, 0) {$#2$};
                \node [style=label] (5) at (1.25, 0) {$#3$};
            }
        \end{pgfonlayer}
        \begin{pgfonlayer}{edgelayer}
            \draw (6.center) to (3);
            \draw (7.center) to (3);
        \end{pgfonlayer}
    \end{tikzpicture}        
}

\NewDocumentCommand{\mapCirc}{moo}{
    \begin{tikzpicture}
        \begin{pgfonlayer}{nodelayer}
            \node [style=map] (3) at (0, 0) {$#1$};
            \node [style=none] (6) at (-1, 0) {};
            \node [style=none] (7) at (1, 0) {};
            \IfNoValueF{#3}{
                \node [style=label] (4) at (-1.25, 0) {$#2$};
                \node [style=label] (5) at (1.25, 0) {$#3$};
            }
        \end{pgfonlayer}
        \begin{pgfonlayer}{edgelayer}
            \draw (6.center) to (3);
            \draw (7.center) to (3);
        \end{pgfonlayer}
    \end{tikzpicture}        
}

\NewDocumentCommand{\mapOpCirc}{moo}{
    \begin{tikzpicture}
        \begin{pgfonlayer}{nodelayer}
            \node [style=mapOp] (3) at (0, 0) {$#1$};
            \node [style=none] (6) at (-1, 0) {};
            \node [style=none] (7) at (1, 0) {};
            \IfNoValueF{#3}{
                \node [style=label] (4) at (-1.25, 0) {$#2$};
                \node [style=label] (5) at (1.25, 0) {$#3$};
            }
        \end{pgfonlayer}
        \begin{pgfonlayer}{edgelayer}
            \draw (6.center) to (3);
            \draw (7.center) to (3);
        \end{pgfonlayer}
    \end{tikzpicture}        
}

\NewDocumentCommand{\sqCirc}{moo}{
    \begin{tikzpicture}
        \begin{pgfonlayer}{nodelayer}
            \node [style=square] (3) at (0, 0) {$#1$};
            \node [style=none] (6) at (-1, 0) {};
            \node [style=none] (7) at (1, 0) {};
            \IfNoValueF{#3}{
                \node [style=label] (4) at (-1.25, 0) {$#2$};
                \node [style=label] (5) at (1.25, 0) {$#3$};
            }
        \end{pgfonlayer}
        \begin{pgfonlayer}{edgelayer}
            \draw (6.center) to (3);
            \draw (7.center) to (3);
        \end{pgfonlayer}
    \end{tikzpicture}        
}

\NewDocumentCommand{\idCirc}{o}{
    \begin{tikzpicture}
        \begin{pgfonlayer}{nodelayer}
            \IfNoValueF{#1}{
                \node [style=label] (4) at (-1, 0) {$#1$};
                \node [style=label] (5) at (1, 0) {$#1$};
            }
            \node [style=none] (6) at (-0.75, 0) {};
            \node [style=none] (7) at (0.75, 0) {};
        \end{pgfonlayer}
        \begin{pgfonlayer}{edgelayer}
            \draw (6.center) to (7.center);
        \end{pgfonlayer}
    \end{tikzpicture}
}

\NewDocumentCommand{\copyCirc}{o}{
    \begin{tikzpicture}
        \begin{pgfonlayer}{nodelayer}
            \IfNoValueF{#1}{
                \node [style=label] (4) at (-1, 0) {$#1$};
                \node [style=label] (5) at (1, 0.5) {$#1$};
                \node [style=label] (10) at (1, -0.5) {$#1$};
            }
            \node [style=none] (6) at (-0.75, 0) {};
            \node [style=black] (7) at (0, 0) {};
            \node [style=none] (8) at (0.65, 0.5) {};
            \node [style=none] (9) at (0.65, -0.5) {};
            \node [style=none] (11) at (0.75, 0.5) {};
            \node [style=none] (12) at (0.75, -0.5) {};
        \end{pgfonlayer}
        \begin{pgfonlayer}{edgelayer}
            \draw (6.center) to (7);
            \draw [bend right] (8.center) to (7);
            \draw [bend right] (7) to (9.center);
            \draw (11.center) to (8.center);
            \draw (12.center) to (9.center);
        \end{pgfonlayer}
    \end{tikzpicture}             
}

\NewDocumentCommand{\discardCirc}{o}{
    \begin{tikzpicture}
        \begin{pgfonlayer}{nodelayer}
            \IfNoValueF{#1}{
                \node [style=label] (4) at (-1, 0) {$#1$};
            }
            \node [style=none] (6) at (-0.75, 0) {};
            \node [style=black] (7) at (0, 0) {};
        \end{pgfonlayer}
        \begin{pgfonlayer}{edgelayer}
            \draw (6.center) to (7);
        \end{pgfonlayer}
    \end{tikzpicture}
}

\NewDocumentCommand{\cocopyCirc}{o}{
    \begin{tikzpicture}
        \begin{pgfonlayer}{nodelayer}
            \IfNoValueF{#1}{
                \node [style=label] (4) at (1, 0) {$#1$};
                \node [style=label] (5) at (-1, 0.5) {$#1$};
                \node [style=label] (10) at (-1, -0.5) {$#1$};
            }
            \node [style=none] (6) at (0.75, 0) {};
            \node [style=black] (7) at (0, 0) {};
            \node [style=none] (8) at (-0.65, 0.5) {};
            \node [style=none] (9) at (-0.65, -0.5) {};
            \node [style=none] (11) at (-0.75, 0.5) {};
            \node [style=none] (12) at (-0.75, -0.5) {};
        \end{pgfonlayer}
        \begin{pgfonlayer}{edgelayer}
            \draw (6.center) to (7);
            \draw [bend left] (8.center) to (7);
            \draw [bend left] (7) to (9.center);
            \draw (11.center) to (8.center);
            \draw (12.center) to (9.center);
        \end{pgfonlayer}
    \end{tikzpicture}              
}

\NewDocumentCommand{\codiscardCirc}{o}{
    \begin{tikzpicture}
        \begin{pgfonlayer}{nodelayer}
            \IfNoValueF{#1}{
                \node [style=label] (4) at (0, 0) {$#1$};
            }
            \node [style=none] (6) at (-0.25, 0) {};
            \node [style=black] (7) at (-1, 0) {};
        \end{pgfonlayer}
        \begin{pgfonlayer}{edgelayer}
            \draw (6.center) to (7);
        \end{pgfonlayer}
    \end{tikzpicture}    
}

\NewDocumentCommand{\symmCirc}{oo}{
    \begin{tikzpicture}
        \begin{pgfonlayer}{nodelayer}
            \IfNoValueF{#2}{
                \node [style=label] (5) at (-0.85, 0.5) {$#1$};
                \node [style=label] (10) at (-0.85, -0.5) {$#2$};
                \node [style=label] (13) at (0.85, 0.5) {$#2$};
                \node [style=label] (14) at (0.85, -0.5) {$#1$};
            }
            \node [style=none] (8) at (-0.5, 0.5) {};
            \node [style=none] (9) at (-0.5, -0.5) {};
            \node [style=none] (11) at (-0.6, 0.5) {};
            \node [style=none] (12) at (-0.6, -0.5) {};
            \node [style=none] (15) at (0.5, 0.5) {};
            \node [style=none] (16) at (0.5, -0.5) {};
            \node [style=none] (17) at (0.6, 0.5) {};
            \node [style=none] (18) at (0.6, -0.5) {};
        \end{pgfonlayer}
        \begin{pgfonlayer}{edgelayer}
            \draw (11.center) to (8.center);
            \draw (12.center) to (9.center);
            \draw (17.center) to (15.center);
            \draw (18.center) to (16.center);
            \draw [in=-180, out=0, looseness=0.75] (8.center) to (16.center);
            \draw [in=0, out=180, looseness=0.75] (15.center) to (9.center);
        \end{pgfonlayer}
    \end{tikzpicture}    
}

\NewDocumentCommand{\seqCirc}{mmoo}{
    \begin{tikzpicture}
        \begin{pgfonlayer}{nodelayer}
            \node [style=box] (3) at (-0.5, 0) {$#1$};
            \node [style=box] (8) at (0.75, 0) {$#2$};
            \node [style=none] (6) at (-1.5, 0) {};
            \IfNoValueF{#4}{
                \node [style=label] (4) at (-1.75, 0) {$#3$};
                \node [style=label] (5) at (2, 0) {$#4$};
            }
            \node [style=none] (9) at (1.75, 0) {};
        \end{pgfonlayer}
        \begin{pgfonlayer}{edgelayer}
            \draw (6.center) to (3);
            \draw (3) to (8);
            \draw (8) to (9.center);
        \end{pgfonlayer}
    \end{tikzpicture}    
}

\NewDocumentCommand{\parCirc}{mmoooo}{
    \begin{tikzpicture}
        \begin{pgfonlayer}{nodelayer}
            \node [style=box] (3) at (0, 0.5) {$#1$};
            \node [style=box] (8) at (0, -0.5) {$#2$};
            \node [style=none] (6) at (-1, 0.5) {};
            \node [style=none] (7) at (1, 0.5) {};
            \IfNoValueF{#4}{
                \node [style=label] (4) at (-1.25, 0.5) {$#3$};
                \node [style=label] (5) at (1.25, 0.5) {$#4$};
            }
            \node [style=none] (9) at (-1, -0.5) {};
            \node [style=none] (10) at (1, -0.5) {};
            \IfNoValueF{#6}{
                \node [style=label] (11) at (-1.25, -0.5) {$#5$};
                \node [style=label] (12) at (1.25, -0.5) {$#6$};
            }
        \end{pgfonlayer}
        \begin{pgfonlayer}{edgelayer}
            \draw (6.center) to (3);
            \draw (7.center) to (3);
            \draw (9.center) to (8);
            \draw (10.center) to (8);
        \end{pgfonlayer}
    \end{tikzpicture}    
}

\NewDocumentCommand{\emptyCirc}{}{
    \begin{tikzpicture}
        \begin{pgfonlayer}{nodelayer}
            \node [style=none] (0) at (-0.5, 0.5) {};
            \node [style=none] (1) at (0.5, 0.5) {};
            \node [style=none] (2) at (0.5, -0.5) {};
            \node [style=none] (3) at (-0.5, -0.5) {};
        \end{pgfonlayer}
        \begin{pgfonlayer}{edgelayer}
            \draw [densely dashed] (3.center)
                 to (0.center)
                 to (1.center)
                 to (2.center)
                 to cycle;
        \end{pgfonlayer}
    \end{tikzpicture}          
}

\NewDocumentCommand{\stateCirc}{mo}{
    \begin{tikzpicture}
        \begin{pgfonlayer}{nodelayer}
            \node [style=box] (3) at (0, 0) {$#1$};
            \node [style=none] (7) at (1, 0) {};
            \IfNoValueF{#2}{
                \node [style=label] (5) at (1.25, 0) {$#2$};
            }
        \end{pgfonlayer}
        \begin{pgfonlayer}{edgelayer}
            \draw (7.center) to (3);
        \end{pgfonlayer}
    \end{tikzpicture}    
}

\NewDocumentCommand{\effectCirc}{mo}{
    \begin{tikzpicture}
        \begin{pgfonlayer}{nodelayer}
            \node [style=box] (3) at (1.25, 0) {$#1$};
            \node [style=none] (7) at (0.25, 0) {};
            \IfNoValueF{#2}{
                \node [style=label] (5) at (0, 0) {$#2$};
            }
        \end{pgfonlayer}
        \begin{pgfonlayer}{edgelayer}
            \draw (7.center) to (3);
        \end{pgfonlayer}
    \end{tikzpicture}    
}

\NewDocumentCommand{\nbox}{moooo}{
\begin{tikzpicture}
	\begin{pgfonlayer}{nodelayer}
		\node [style=box] (0) at (0, 0) {$#1$};
		\node [style=label] (1) at (1.25, 0.25) {$#4$};
		\node [style=label] (2) at (1.25, -0.5) {$#5$};
		\node [style=label] (3) at (-1.25, 0.25) {$#2$};
		\node [style=label] (4) at (-1.25, -0.5) {$#3$};
		\node [style=label] (6) at (-0.75, 0.0) {$\vdots$};
		\node [style=label] (7) at (0.75, 0.0) {$\vdots$};
	\end{pgfonlayer}
	\begin{pgfonlayer}{edgelayer}
		\draw [style=wire, bend left=345, looseness=1.25] (1) to (0);
		\draw [style=wire, bend left=15] (2) to (0);
		\draw [style=wire, bend left=15, looseness=1.25] (3) to (0);
		\draw [style=wire, bend right=15, looseness=1.25] (4) to (0);
	\end{pgfonlayer}
\end{tikzpicture}
}

\usepackage{amssymb,amsmath,amscd,mathrsfs,amsfonts,bm,mathtools}
\usepackage{hyperref}

\usepackage{cleveref}

\usepackage{algorithm}
\usepackage{algorithmic}
\input xy 
\xyoption{all}

\usepackage{color}

\usepackage[normalem]{ulem}

\DeclareMathOperator*{\argmin}{arg\,min}

\usepackage{newfloat}
\usepackage{listings}
\floatstyle{ruled}
\newfloat{listing}{tb}{lst}{}
\floatname{listing}{Listing}

\usepackage{tikz}
\usepackage{comment}

\renewcommand{\proofname}{\textit{Proof.}} 

\newcommand{\id}{\mathrm{id}}

\newcommand{\cost}{\mathrm{cost}}

\newcommand{\sem}[1]{[\![#1]\!]_{\mathcal{I}}}
\newcommand{\com}[1]{\langle\!\langle#1\rangle\!\rangle_{\mathbb{O}}}

\newcommand{\mb}{\mathbb }
\newcommand{\mc}{\mathcal }
\newcommand{\B}{\mathbf }
\newcommand{\N}{{\mathbb{N}}}
\newcommand{\R}{{\mathbb{R}}}

\newcommand{\vertiii}[1]{{\left\vert\kern-0.25ex\left\vert\kern-0.25ex\left\vert #1
\right\vert\kern-0.25ex\right\vert\kern-0.25ex\right\vert}}

\newcommand{\defeq}{:=}
\newcommand{\Cat}[1]{{ \bf #1}}

\begin{document}
%
\title{Mathematical Foundation of \\ Interpretable Equivariant  Surrogate Models}
%
%


\author{Jacopo Joy Colombini\inst{1} \and Filippo Bonchi\inst{2} \and
Francesco Giannini\inst{1}
\and
Fosca Giannotti\inst{1}
\and
Roberto Pellungrini\inst{1}
\and
Patrizio Frosini\inst{2}
}
\authorrunning{Colombini et al.}
%
\institute{University of Pisa, Pisa, Italy\\ \email{\{
name.surname
\}@unipi.it} \and
Scuola Normale Superiore, Pisa, Italy\\
\email{\{
name.surname
\}@sns.it}}

\maketitle              
\begin{abstract}
This paper introduces a rigorous mathematical framework for neural network explainability, and more broadly for the explainability of equivariant operators called Group Equivariant Operators (GEOs)
based on Group Equivariant Non-Expansive Operators (GENEOs) transformations. The central concept involves quantifying the distance between GEOs by measuring the non-commutativity of specific diagrams. Additionally, the paper proposes a definition of interpretability of GEOs according to a complexity measure that can be defined according to each user preferences. Moreover, we explore the formal properties of this framework and show how it can be applied in classical machine learning scenarios, like image classification with convolutional neural networks.
\end{abstract}

\section{Introduction}
\label{sec:intro}

What is an ``\textit{explanation}''? 
An explanation can be seen as a combination of elementary blocks, much like a sentence is formed by words, a formula by symbols, or a proof by axioms and lemmas. The key question is when such a combination effectively explains a phenomenon. Notably, the quality of an explanation is observer-dependent—what is clear to a scientist may be incomprehensible to a philosopher or a child. In our approach, an explanation of a phenomenon $P$ is convenient for an observer $\mathbb{O}$ if (i) $\mathbb{O}$ finds it comfortable, meaning the building blocks are easy to manipulate, and (ii) it is convincing, meaning $\mathbb{O}$ perceives $P$ and the explanation as sufficiently close.
We contextualize this perspective by assuming that the phenomenon is an AI agent, viewed as an operator, thus saying that the action of an agent $\mathbb{A}$ is explained by another agent $\mathbb{B}$ from the perspective of an observer $\mathbb{O}$ if:
\begin{enumerate} 
\item $\mathbb{O}$ perceives $\mathbb{B}$ as close to $\mathbb{A}$; 
\item $\mathbb{O}$ perceives $\mathbb{B}$ as less complex than $\mathbb{A}$. 
\end{enumerate} 
This is represented in \Cref{fig:sketch} where we show this concept with an example.
Notwithstanding the fact that agent $\mathbb{O}$ has the right to choose subjective criteria to measure how good $\mathbb{B}$ is to approximate $\mathbb{A}$ and their complexities, this paper introduces a mathematical framework for these measurements.
\begin{figure}[t]
    \centering
\begin{subfigure}[t]{0.45\linewidth}
    \includegraphics[width=\linewidth]{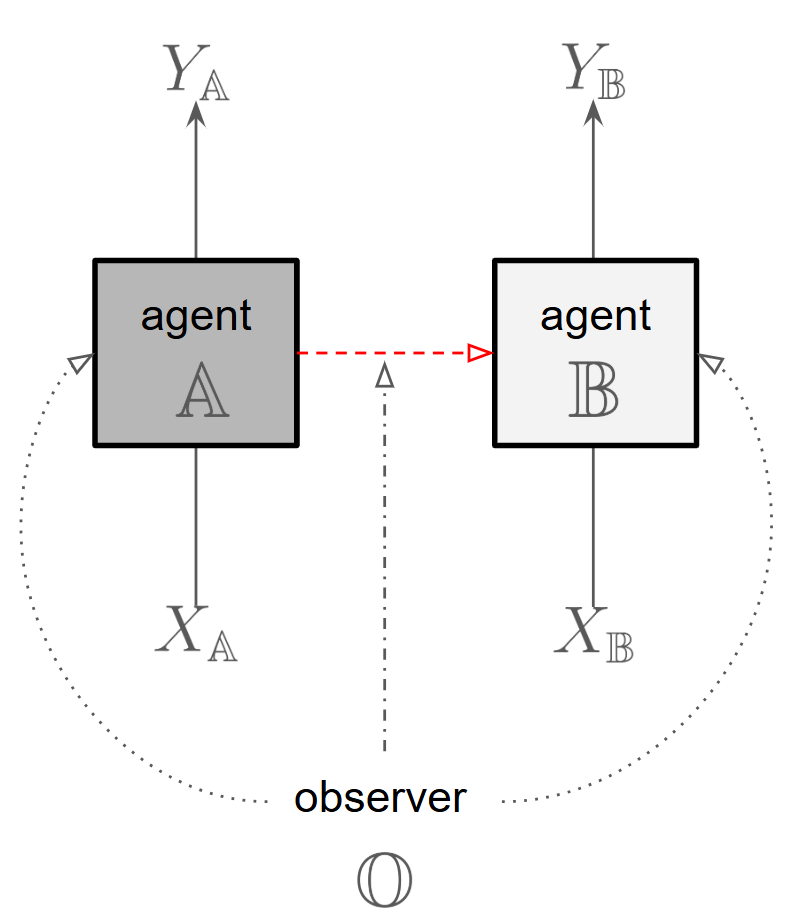}
    \caption{An agent $\mathbb{B}$ explains an agent $\mathbb{A}$ for an observer $\mathbb{O}$: $\mathbb{O}$ perceives $\mathbb{B}$ as close to $\mathbb{A}$, and $\mathbb{B}$ as less opaque than $\mathbb{A}$}
    \end{subfigure} \hfill
    \begin{subfigure}[t]{0.45\linewidth}
    \includegraphics[width=\linewidth]{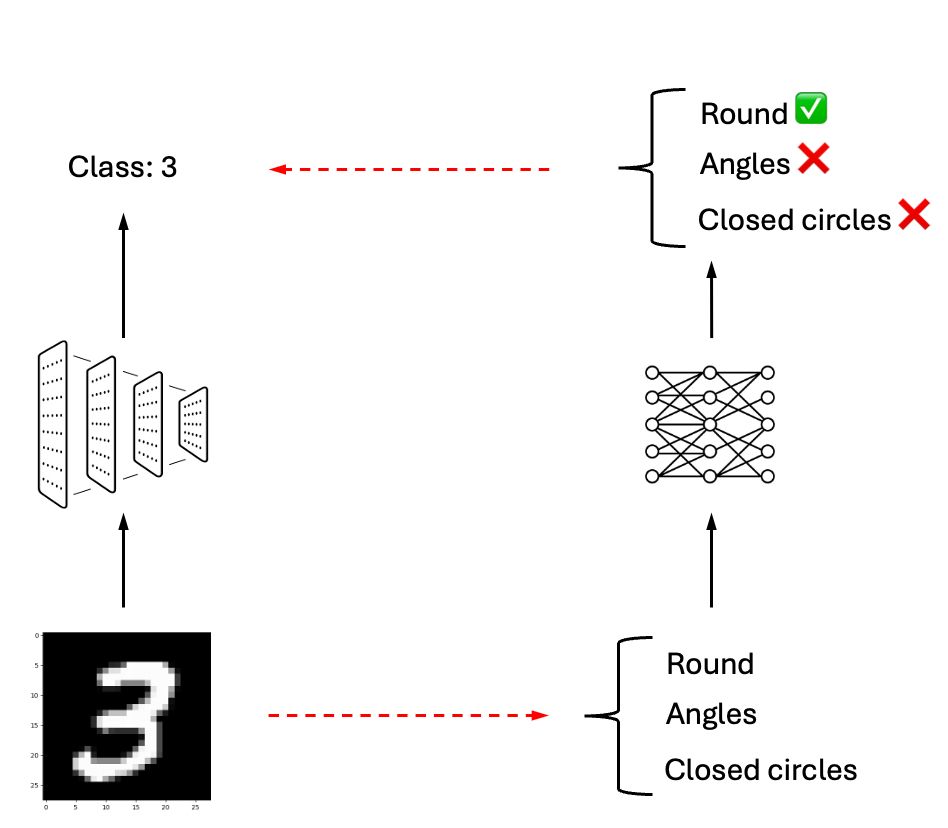}
    \caption{A practical example: an image classification task can be surrogated by a model based on textual concepts}
    \end{subfigure}
    \caption{Representation of interpretable surrogate models exemplified on MNIST.}
    \label{fig:sketch}
\end{figure}

The growing use of complex neural networks in critical applications demands both high performance and transparency in decision-making. While AI interpretability and explainability have advanced, a rigorous mathematical framework for defining and comparing explanations is still lacking \cite{palacio2021xai}. Recent efforts to formalize explanations \cite{giannini2024categorical} and interpretable models \cite{marconato2023interpretability} do not provide practical guidelines for designing or training explainable models, nor do they incorporate the notion of an observer within the theory. 
Moreover, researchers emphasize the importance of Group Equivariant Operators (GEOs) in machine learning \cite{AnRoPo16,cohen2016group,worrall2017harmonic,GeArCaal23}, as they integrate prior knowledge and enhance neural network design control \cite{Bengio2013}. While standard neural networks are universal approximators \cite{cybenko1989approximation}, this typically increases complexity. However, no existing XAI techniques address explaining an equivariant model using another equivariant model.
This paper addresses this gap by introducing a framework for learning interpretable surrogate models of a black-box and defining a measure of interpretability based on an observer’s subjective preferences.
Given the importance of equivariant operators, our XAI framework is built on the theory of GEOs and Group Equivariant Non-Expansive Operators (GENEOs) \cite{FrJa16}. GEOs, a broader class than standard neural networks, are well-suited for processing data with inherent symmetries. Indeed, equivariant networks, such as convolutional \cite{gerken2023geometric} and graph neural networks \cite{satorras2021n}, have proven effective across different tasks \cite{ruhe2023clifford}.
Using GENEO-based transformations, we develop a theory for learning surrogate models of a given GEO by minimizing algebraic diagram commutation errors. The learned surrogate model can either perform a task or approximate a black-box model’s predictions while optimizing interpretability based on an observer’s perception of complexity, while allowing different observers to have distinct interpretability preferences for the same model architecture.

\textbf{Contributions.}  Our contributions can be summarized as follows.

\begin{itemize} \item Introduction of a mathematical framework to define interpretable surrogate models, where interpretability depends on a specific observer.
\item Definition of a distance between GEOs using diagram non-commutativity, providing a quantitative method for model comparison and training.
\item Formal definition of GEOs' complexity to assess model interpretability.
\item We show empirically that these metrics enable training of more interpretable models, usable for direct task-solving or as surrogates for black-box models.
\end{itemize}

The paper is organized as follows. \Cref{sec:back} recalls basics from different Mathematics areas that we use to define our metrics in \Cref{sec:diagrams}.
\Cref{sec:framework} shows how these metrics are used in practice to define a learning problem for an interpretable surrogate model. We show how the proposed framework can be used via an experimental evaluation in \Cref{sec:exp}. Finally, \Cref{sec:relwork} comments on related work and \Cref{sec:conc} draws conclusions and remarks on future work. The Appendix contains additional material and all proofs. 

\section{Mathematical Preliminaries}
\label{sec:back}


The framework proposed in this paper is founded on mathematical structures studied in various fields, such as geometry and category theory. Metric spaces and groups are used to define GE(NE)Os, while categories to compose them.


\subsection{Perception Spaces and GE(NE)Os}
\label{subsec:geneo}

Recall that a \emph{pseudo-metric space} is a pair $(X,d)$ where $X$ is a set and $d\colon X\times X \to[0,\infty]$ is a pseudo-metric, namely a function such that, for all $x,y,z \in X$,
\begin{equation*}
(R)\, d(x,x) = 0\text{,} \qquad 
(S)\, d(x,y) = d(y,x)\text{,} \qquad
(T)\, d(x,z) \leq d(x,y) + d(y,z)\text{.} 
\end{equation*}
A \emph{metric} $d$ is a pseudo-metric that additionally satisfies $d(x,y)=0 \implies x=y$. 
$d\colon X\times X \to[0,\infty]$ is a \textit{hemi-metric} if it only satisfies (R) and (T). We use the informal term distance to refer to either metrics, pseudo-metrics or hemi-metrics.
%

A \emph{group} $\mathbf{G}=(G,\circ, \id_G)$ consists of a set $G$, an associative operation $\circ\colon G\times G \to G$ having a unit element $\id_G \in G$ such that, for all $g\in G$, there exists $g^{-1}\in G$ satisfying $g\circ g^{-1}=g^{-1}\circ g = \id_G$. A \emph{group homomorphism} $T\colon (G,\circ_G, \id_G) \to (K,\circ_K,\id_K)$ is a function $T\colon G\to K$ such that, for all $g_1,g_2\in G$, $T(g_1\circ_G g_2)=T(g_1)\circ_K T(g_2)$.
Given a group $(G,\circ, \id_G)$ and a set $X$, a \emph{group left action} is a function $*\colon G \times X \to X$ such that, for all $x\in X$ and $g_1,g_2\in G$, 
\begin{equation*}
\id_G * x=x \qquad \text{ and } \qquad
(g_1\circ g_2) * x=g_1*(g_2*x)\text{.}
\end{equation*}


\smallskip

With these ingredients, we can now illustrate the notions of perception space GEO and GENEO. We refer the interested reader to \cite{FrJa16,BeFrGiQu19}
and \cite{newmet2018,genperm2023,Boeal24} for a more extensive description of GENEOs and their applications.

\begin{definition}
\label{defpp}
An (extended) \textit{perception space} $(X,d_X,\B G, *)$, shortly $(X,\B G)$, consists of a pseudo-metric space $(X, d_X)$, a group $\B G$,
a left group action $*\colon G\times X \to X$  
such that, for all $x_1,x_2\in X$ and every $g\in G$,
\begin{equation*}d_X(g*x_1,g*x_2)=d_X(x_1,x_2)\text{.}
\end{equation*}
\end{definition}
\begin{example}
   $(X,\B G)$, where $\B G$ is the group containing the rotation of $0^\circ, 90^\circ, 180^\circ, 270^\circ$, and $X$ is a set of images closed under the actions of $\B G$, is a perception space.
\end{example}
Notice that in any perception space, one can define a pseudo-metric over the group $\B G$ by fixing $d_G(g_1,g_2):=\sup_{x\in X}d_X(g_1*x,g_2*x)$ for any $g_1,g_2\in G$. With this definition, one can easily show that $\B G$ is a topological group and that the action $*$ is continuous (see \Cref{prop:topo} in \Cref{app:proofs}).



\begin{definition}
Let $(X,G)$, $(Y,K)$ be two (extended) perception spaces, $f\colon X\to Y$ and $t\colon G\to K$ a group homomorphism. We say that 
$(f,t)$ is an (extended) \emph{group equivariant operator} (GEO)
if $g(g*x)=t(g)*f(x)$ for every $x\in X$, $g\in G$. 
$(f,t)$ is said an (extended) \emph{group equivariant non-expansive operator} (GENEO) in case it is a GEO and it is also non-expansive, i.e.,
\begin{enumerate}
    \item $d_Y(f(x_1),f(x_2))\le d_X(x_1,x_2)$ for every $x_1,x_2\in X$,
    \item $d_K(t(g_1),t(g_2))\le d_G(g_1,g_2)$ for every $g_1,g_2\in G$.
\end{enumerate}
\end{definition}
%
%
The previous extended definitions generalize original perception pairs, GEOs, and GENEOs beyond data represented as functions. We simply refer to them as perception space, GEO, and GENEO. With slight abuse of notation, we use $d_{\mathrm{dt}}$ for the metric  $d_X$ on the set of data, and $d_{\mathrm{gr}}$ for the metric $d_G$ on the group $G$, relying on context to specify the perception space $(X,G)$ under consideration.

\begin{example}[Neural Networks as GEOs]
Neural networks are a special case of GEOs, with different architectures equivariant to specific groups. Convolutional Neural Networks (CNNs) are equivariant to translations, while Graph Neural Networks (GNNs) respect graph permutations. Although standard Multi-Layer Perceptrons are not typically equivariant, they can be viewed as GEOs on the trivial group $\B 1$, containing only the neutral element.
\end{example}

\begin{example}
Let $X_\alpha$ be the set of all subsets $\R^3$ and the group $\B G_\alpha$ the group of all translations in $\R^3$, and let $\tau_{(x,y,z)}$ represent the translations by $(x,y,z)$. Similarly define $X_\beta$ and $\B G_\beta$ in $\R^2$ with $\tau_{(x,y)}$ translating by $(x,y)$. A GENEO $(f,t)$ can be defined where $f(x)$ gives the shadow (orthogonal projection) of $x$ in $X_\beta$ and the homomorphism $t\colon \B G_\alpha\to \B G_\beta$ is given by $t(\tau_{(x,y,z)})=\tau_{(x,y)}$ for projections onto the $xy$-plane. Similarly, defining $t(\tau_{(x,y,z)})=\tau_{(y,z)}$ gives a GENEO for projections onto the  $yz$-plane.
\end{example}


\subsection{A Categorical Algebra of GEOs}
\label{subsec:cat}



We introduce a simple language to specify combinations of GEOs. 
Our proposal rely on the algebra of monoidal categories (CD-categories~\cite{ChoJ19}) that enjoy an intuitive --but formal-- graphical representation by means of string diagrams~\cite{selinger2010survey}.

\paragraph{Syntax.} 
We fix a set $\mathcal{S}$ of basic sorts and we consider the set $\mathcal{S^*}$ of words over $\mathcal{S}$: we write $1$ for the empty word and $U\otimes V$, or just $UV$, for the concatenation of any two words $U,V\in \mathcal{S^*}$. Moreover, we fix a set $\Gamma$ of operator symbols and two functions  $ar, coar \colon \Gamma \to \mathcal{S^*}$. For an operator symbol $g\in \Gamma$, $ar(g)$ represents its arity, intuitively the types of its input and $coar(g)$ its coarity, intuitively its output. The tuple $(\mathcal{S}, \Gamma, ar, coar)$, shortly $\Gamma$, is what is called in categorical jargon a \emph{monoidal signature}. 

We consider terms generated by the following context-free grammar
\[\arraycolsep=1.4pt
\begin{array}{rccccccccccccc}
c &::= & \nbox{g}[A_1][A_n][B_1][B_m] & \mid & \emptyCirc{} & \mid & \idCirc[A] & \mid  & \symmCirc[A][B] & \mid & \copyCirc[A] & \mid & \discardCirc[A] & \mid  \\ 
& & c_1 \circ c_2 & \mid & c_1 \otimes c_2
\end{array}\]
where $A,B, A_i,B_i$ are sorts in $\mathcal{S}$ and $g$ is a symbol in $\Gamma$ with arity $A_1\otimes  \dots \otimes A_n$ and coarity $B_1\otimes  \dots \otimes B_m$. Terms of our grammar can be thought of as circuits where information flows from left to right: the wires on the left represent the input ports, those on the right the outputs; the labels on the  wires specify the types of the ports. The input type of a term is the word in $\mathcal{S}^*$ obtained by reading from top to bottom the labels on the input ports; Similarly for the outpus. The circuit $\nbox{g}[A_1][A_n][B_1][B_m]$ takes $n$ inputs of type $A_1, \dots, A_n$ and produce $m$ outputs of type $B_1, \dots, B_m$; $\emptyCirc{}$ is the empty circuits with no inputs and no output; $\idCirc[A]$ is the wires where information of type $A$ flows from left to right; $\symmCirc[A][B]$ allows for crossing of wires; $\copyCirc[A]$ receives some information of type $A$ and emit two copies as outputs; $\discardCirc[A]$ receives an information of type $A$ and discard it. For arbitrary circuits $c_1$ and $c_2$, $c_1 \circ c_2$ and  $c_1 \otimes c_2$ represent, respectively their sequential and parallel composition drawn as
\begin{center}
\begin{tabular}{ccc}
\begin{tikzpicture}
	\begin{pgfonlayer}{nodelayer}
		\node [style=box] (0) at (2.5, 0) {$c_1$};
		\node [style=label] (1) at (3.75, 0.25) {$C_1$};
		\node [style=label] (2) at (3.75, -0.5) {$C_o$};
		\node [style=label] (6) at (1.75, 0) {$\vdots$};
		\node [style=label] (7) at (3.25, 0) {$\vdots$};
		\node [style=box] (8) at (0, 0) {$c_2$};
		\node [style=label] (11) at (-1.25, 0.25) {$A_1$};
		\node [style=label] (12) at (-1.25, -0.5) {$A_n$};
		\node [style=label] (13) at (-0.75, 0) {$\vdots$};
		\node [style=label] (14) at (0.75, 0) {$\vdots$};
		\node [style=label] (15) at (1.25, 0.75) {};
		\node [style=label] (16) at (1.25, 0.75) {$B_1$};
		\node [style=label] (17) at (1.25, -0.25) {};
		\node [style=label] (18) at (1.25, -0.25) {$B_m$};
	\end{pgfonlayer}
	\begin{pgfonlayer}{edgelayer}
		\draw [style=wire, bend left=345, looseness=1.25] (1) to (0);
		\draw [style=wire, bend left=15] (2) to (0);
		\draw [style=wire, bend left=15, looseness=1.25] (11) to (8);
		\draw [style=wire, bend right=15, looseness=1.25] (12) to (8);
		\draw [style=wire, bend left] (8) to (0);
		\draw [style=wire, bend right, looseness=1.25] (8) to (0);
	\end{pgfonlayer}
\end{tikzpicture}
\qquad
&
\qquad
 \text{ and } 
&
\qquad
\begin{tikzpicture}
	\begin{pgfonlayer}{nodelayer}
		\node [style=box] (0) at (0, 0.75) {$c_1$};
		\node [style=label] (1) at (1.25, 1) {$B_1$};
		\node [style=label] (2) at (1.25, 0.25) {$B_m$};
		\node [style=label] (3) at (-1.25, 1) {$A_1$};
		\node [style=label] (4) at (-1.25, 0.25) {$A_n$};
		\node [style=label] (6) at (-0.75, 0.75) {$\vdots$};
		\node [style=label] (7) at (0.75, 0.75) {$\vdots$};
		\node [style=box] (8) at (0, -0.5) {$c_2$};
		\node [style=label] (9) at (1.25, -0.25) {$D_1$};
		\node [style=label] (10) at (1.25, -1) {$D_k$};
		\node [style=label] (11) at (-1.25, -0.25) {$C_1$};
		\node [style=label] (12) at (-1.25, -1) {$C_j$};
		\node [style=label] (13) at (-0.75, -0.5) {$\vdots$};
		\node [style=label] (14) at (0.75, -0.5) {$\vdots$};
	\end{pgfonlayer}
	\begin{pgfonlayer}{edgelayer}
		\draw [style=wire, bend left=345, looseness=1.25] (1) to (0);
		\draw [style=wire, bend left=15] (2) to (0);
		\draw [style=wire, bend left=15, looseness=1.25] (3) to (0);
		\draw [style=wire, bend right=15, looseness=1.25] (4) to (0);
		\draw [style=wire, bend left=345, looseness=1.25] (9) to (8);
		\draw [style=wire, bend left=15] (10) to (8);
		\draw [style=wire, bend left=15, looseness=1.25] (11) to (8);
		\draw [style=wire, bend right=15, looseness=1.25] (12) to (8);
	\end{pgfonlayer}
\end{tikzpicture}\text{ .}
\end{tabular}
\end{center}
As expected, the sequential composition of $c_1$ and $c_2$ is possible only when the outputs of $c_2$ coincides with the inputs of $c_1$.

\begin{remark}
The reader may have noticed that different syntactic terms are rendered equal by the diagrammatic representation. For instance both $c_1 \circ (c_2 \circ c_3)$ and $(c_1 \circ c_2) \circ c_3$ are drawn as
\begin{center}\begin{tikzpicture}
	\begin{pgfonlayer}{nodelayer}
		\node [style=box] (0) at (2.5, 0) {$c_1$};
		\node [style=label] (6) at (1.75, 0) {$\vdots$};
		\node [style=box] (8) at (0, 0) {$c_3$};
		\node [style=label] (11) at (-1.25, 0.25) {$A_1$};
		\node [style=label] (12) at (-1.25, -0.5) {$A_n$};
		\node [style=label] (13) at (-0.75, 0) {$\vdots$};
		\node [style=label] (14) at (0.75, 0) {$\vdots$};
		\node [style=label] (15) at (1.25, 0.75) {};
		\node [style=label] (16) at (1.25, 0.75) {$B_1$};
		\node [style=label] (17) at (1.25, -0.25) {};
		\node [style=label] (18) at (1.25, -0.25) {$B_m$};
		\node [style=box] (19) at (5, 0) {$c_1$};
		\node [style=label] (20) at (6.25, 0.25) {$D_1$};
		\node [style=label] (21) at (6.25, -0.5) {$D_p$};
		\node [style=label] (22) at (4.25, 0) {$\vdots$};
		\node [style=label] (23) at (5.75, 0) {$\vdots$};
		\node [style=box] (24) at (2.5, 0) {$c_2$};
		\node [style=label] (25) at (3.25, 0) {$\vdots$};
		\node [style=label] (26) at (3.75, 0.75) {};
		\node [style=label] (27) at (3.75, 0.75) {$C_1$};
		\node [style=label] (28) at (3.75, -0.25) {};
		\node [style=label] (29) at (3.75, -0.25) {$C_o$};
	\end{pgfonlayer}
	\begin{pgfonlayer}{edgelayer}
		\draw [style=wire, bend left=15, looseness=1.25] (11) to (8);
		\draw [style=wire, bend right=15, looseness=1.25] (12) to (8);
		\draw [style=wire, bend left] (8) to (0);
		\draw [style=wire, bend right, looseness=1.25] (8) to (0);
		\draw [style=wire, bend left=345, looseness=1.25] (20) to (19);
		\draw [style=wire, bend left=15] (21) to (19);
		\draw [style=wire, bend left] (24) to (19);
		\draw [style=wire, bend right, looseness=1.25] (24) to (19);
	\end{pgfonlayer}
\end{tikzpicture}
\end{center}
This is not an issue since the two terms represent the same GEO via the semantics that we illustrate here below, after a minimal background on categories.
\end{remark}

\paragraph{Categories.} Diagrams are arrows of the (strict) CD category freely generated by the monoidal signature $\Gamma$. The reader who is \emph{not} an expert in category theory may safely ignore this fact and only know that a \emph{category} $\Cat{C}$ consists of (1) a collection of objects denoted by $Ob({\bf C})$; (2) for all objects $A,B\in Ob(\Cat{C})$, a collection of arrows $f\colon A \to B$ with source object $A$ and target object $B$; (3) for all objects $A$, an identity arrow $id_{A}\colon A \to A$ and (4) for all arrows $f\colon A \to B$ and $g\colon B \to C$,  a composite arrow $g\circ f \colon A \to C$ satisfying
\[f\circ (g\circ h)=(f\circ g)\circ h \qquad f\circ id_{A} = f=id_{B}\circ f\]
for all $f\colon A \to B$, $g\colon B\to C$ and $h\colon D\to E$. 

Three categories will be particularly relevant for our work: the category
$\Cat{Diag}_{\Gamma}$ having words in $\mathcal{S}^*$ as objects and diagrams as arrows, the category
$\Cat{GEO}$ having perception spaces as objects and GEOs as arrows and the category $\Cat{GENEO}$ having perception spaces as objects and GENEOs as arrows. 

\paragraph{Semantics.} As mentioned at the beginning of this section, our diagrammatic language allows one to express combinations of GEOs. Intuitively, the symbols in $\Gamma$ are basic \emph{building blocks} that can be composed in sequence and in parallel with the aid of some wiring technology. The building blocks have to be thought of as atomic GEOs, while diagrams as composite ones.

To formally provide semantics to diagrams in terms of GEOs, the key ingredient is an \emph{interpretation} $\mathcal{I}$ of the monoidal signature $\Gamma$ within the (monoidal) category $\Cat{GEO}$, shortly, a function assigning to each symbol $g\in \Gamma$ a corresponding GEO. Then, by means of a universal property (or, depending on one's perspective, abstract mumbo jumbo), one obtains a function (actually a functor) $\sem{-}\colon {\bf Diag}_\Gamma \to {\bf GEO}$ assigning to each diagram the denoted GEOs (see Table~\ref{tab:semantics} in the Appendix for a simple inductive definition).

Note that $\sem{-}$ may not be surjective, in the sense that not all GEOs
are denoted by some diagrams: we call $\mathcal{G}^\Gamma_{\mathcal{I}}$ the image of ${\bf Diag}_\Gamma$ through $\sem{-}$, i.e., \[\mathcal{G}^\Gamma_{\mathcal{I}}:=\{(f,t) \;|\; \exists c\in \Cat{Diag}_\Gamma \text{ s.t. } \sem{c}=(f,t)\}\text{.}\] 
Hereafter, we fix a monoidal signature $\Gamma$ and an interpretation $\mathcal{I}$ and we write $\mathcal{G}^\Gamma_{\mathcal{I}}$ simply as $\mathcal{G}$. This represents the universe of GEOs that are interesting for the observer, which we are going to introduce in the next section.

\section{Observers-based Approximation and Complexity}
\label{sec:diagrams}


This paper aims at developing an applicable mathematical theory of interpretable models, which is based on the following intuition: an agent $\mathbb{A}$ can be interpreted via another agent $\mathbb{B}$ from the perspective of an observer $\mathbb{O}$ if: i) $\mathbb{O}$ perceives $\mathbb{B}$ as similar to $\mathbb{A}$ and ii) $\mathbb{O}$ perceives $\mathbb{B}$ as less complex than $\mathbb{A}$. 
This perspective motivates us to build a framework allowing the modeling of distance measures for GEOs (\Cref{subsec:dist}) and their degree of complexity (opaqueness/not interpretability, \Cref{subsec:complexity}), w.r.t. the specification of a certain observer. 

\begin{definition}
    An \emph{observer} $\mathbb{O}$ interested in $\mathcal{G}$ is a couple $(\Cat{T},\mc{C})$ where:
    \begin{itemize}
   \item $\Cat{T}$ is a category of \emph{translations GENEOs}, namely a category having as objects $Ob(\Cat{T})$ those perception spaces that are sources and targets of GEOs in $\mathcal{G}$ and as arrows $Hom(\Cat{T})$ a selected set of GENEOs. 
        \item $\mc{C}$ is a \emph{complexity assignment}, namely a function $\mc{C} \colon \Gamma \to \mathbb{R}^+$.
    \end{itemize}
\end{definition}

The \emph{translation GENEOs} in $\mathbf{T}$ describe all the possible ways that the observer can ``translate'' data belonging to one perception space into data belonging to another perception space. Requiring these to be GENEOs, i.e., non-expansive, ensures that such translations performed by the observer cannot enlarge distances between data.
For example, the observer may admit only isometries as morphisms in $\mathbf{T}$, or the observer may not admit any translation at all, meaning that $\mathbf{T}$ only contains identities (note that this is the smallest possible $\Cat{T}$).

The \emph{complexity assignment} $\mc{C} \colon \Gamma \to \mathbb{R}^+$ maps any building block $g$ from $\Gamma$ into a positive real number, a quantity that represent how \emph{complex} is perceived $g$ by the observer. Here complexity does not refer to the usual computational complexity but rather to the degree of \emph{stress} that the observer perceives in dealing with $g$. Note that such assignment is completely arbitrary and thus, different observers may assign different complexities to the same building block. Any observer can specify what are the types of functions that he deems interpretable and/or more informative, from their perspective, for a given problem.

\subsection{Surrogate Distance of GEOs}
\label{subsec:dist}
To formalize the notion of a surrogate model for an observer $\mb{O}$, we introduce a new hemi-metric $h_{\mb{O}}$, which we call the \emph{surrogate distance} of a GEO for another GEO. To proceed, it is fundamental the notion of \textit{crossed translation pair}.

\begin{definition}
Let  $(f_\alpha,t_\alpha)\colon (X_\alpha,G_\alpha)\to (Y_\alpha,K_\alpha)$ and $(f_\beta,t_\beta)\colon(X_\beta,G_\beta)\to (Y_\beta,K_\beta)$ be two GEOs in $\mathcal{G}$. 
A \emph{crossed pair of translation} $\pi$ from $(f_\alpha,t_\alpha)$ to $(f_\beta,t_\beta)$, written $\pi\colon (f_\alpha,t_\alpha) \leftrightharpoons_{\mathbf{T}} (f_\beta,t_\beta)$, is a couple $\Big((l_{\alpha,\beta},p_{\alpha,\beta}),(m_{\beta,\alpha},q_{\beta,\alpha})\Big)$ where
\begin{itemize}
    \item $(l_{\alpha,\beta},p_{\alpha,\beta})\colon (X_\alpha,G_\alpha) \to (X_\beta,G_\beta)$ is a GENEO in $\Cat{T}$,
    \item $(m_{\beta,\alpha},q_{\beta,\alpha}) \colon (Y_\beta,K_\beta) \to (Y_\alpha,K_\alpha)$ is a GENEO in $\Cat{T}$.
\end{itemize}
\end{definition}
Figure \ref{fig:diagram} provides an intuitive visualization of a crossed pair of translation GENEOs.
Note that the two GENEOs have opposite directions.

\begin{figure}[t]
\begin{center}
\includegraphics[width=0.6\textwidth]{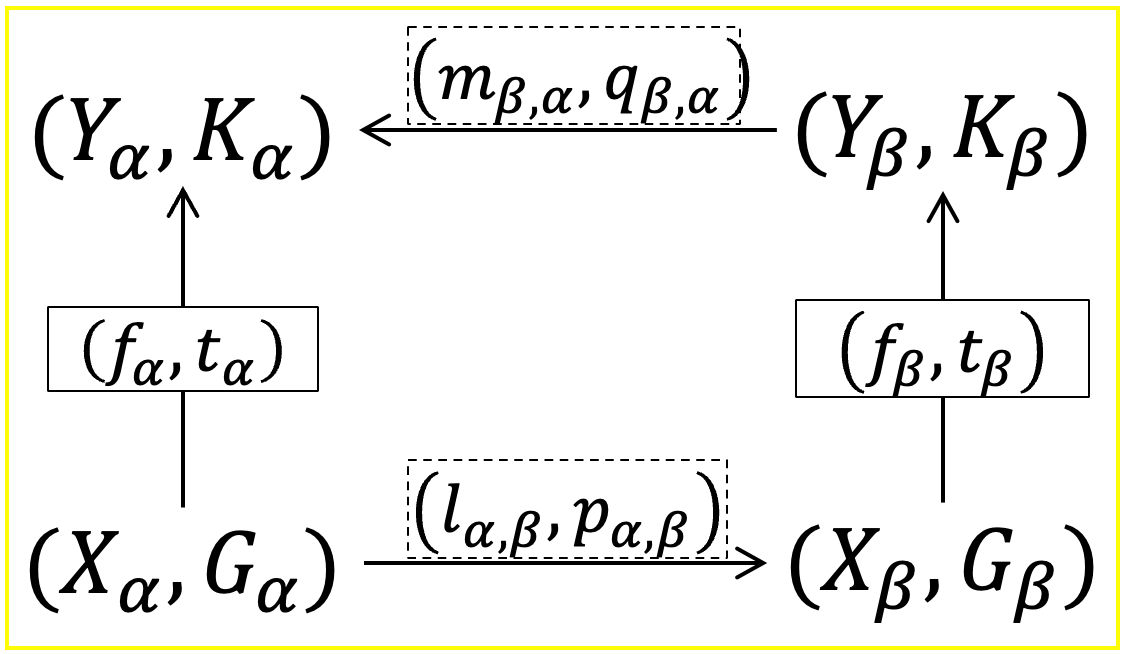}
\caption{Example of a crossed translation pair $\pi\colon (f_\alpha,t_\alpha) \leftrightharpoons_{\mathbf{T}} (f_\beta,t_\beta)$. We distinguish by solid and dashed blocks the GEOs in $\mc{G}$ from the GENEOs in $Hom(\Cat{T}$.}
\label{fig:diagram}
\end{center}
\end{figure}

Next, we define the cost of a crossed translation pair.
\begin{definition}
Let $\pi=\Big((l_{\alpha,\beta},p_{\alpha,\beta}),(m_{\beta,\alpha},q_{\beta,\alpha})\Big)$ be a crossed translation pair from $(f_\alpha,t_\alpha)\colon (X_\alpha,G_\alpha)\to (Y_\alpha,K_\alpha)$ to $(f_\beta,t_\beta)\colon (X_\beta,G_\beta)\to (Y_\beta,K_\beta)$.
The \emph{functional cost} of $\pi$, written $\cost(\pi)$, is defined as follows.
\begin{equation}
\label{eq:loss}
    \cost(\pi) := \frac{1}{|X_\alpha|}\sum_{x \in X_\alpha}
d_{\mathrm{dt}}\Big((m_{\beta,\alpha}\circ f_\beta\circ l_{\alpha,\beta})(x),
f_\alpha(x)\Big)
\end{equation} 
\end{definition}

\begin{remark}
Note that in \Cref{eq:loss}, $|X_\alpha|$ denotes the cardinality of the set $X_\alpha$. Whenever such set is infinite the cost is not defined. Although this never happens in practical cases, one can easily generalize \eqref{eq:loss} to deal with infinite sets by enriching $X_\alpha$ with a Borel probability measure: see \eqref{eq:loss2} in the Appendix.
\end{remark}

Intuitively, the value $\cost(\pi)$ measures the distance of the two paths in the diagram in Figure \ref{fig:dist}. With this, one can easily define a distance between GEOs. 

\begin{figure}[t]
\begin{center}
\includegraphics[width=0.6\textwidth]{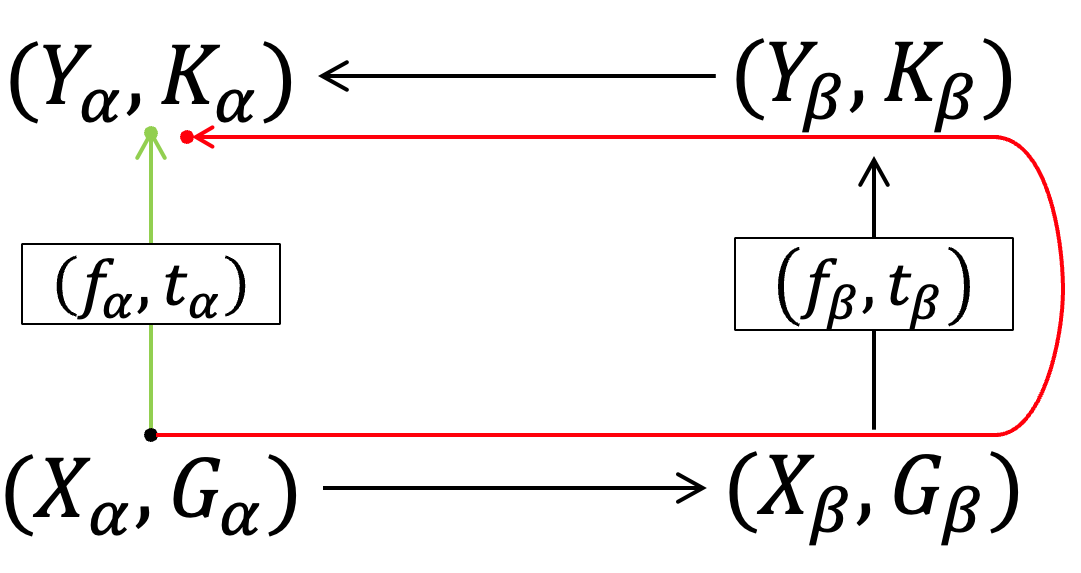}
\caption{The surrogate distance measures how far the diagram is to commute.}
\label{fig:dist}
\end{center}
\end{figure}

\begin{definition}\label{E-distance}
Let $(f_\alpha,t_\alpha)$ and $(f_\beta,t_\beta)$ be two GEOs in $\mathcal{G}$.  
The \emph{surrogate distance} of $(f_\beta,t_\beta)$ from $(f_\alpha,t_\alpha)$, written $h_{\mb{O}}\Big((f_\alpha,t_\alpha),(f_\beta,t_\beta)\Big)$, is defined as  
\begin{equation}
\label{eq:semi-metric}
    \inf \{\cost(\pi) \mid \pi\colon (f_\alpha,t_\alpha)\leftrightharpoons_{\mathbf{T}}(f_\beta,t_\beta)\}
\end{equation}
\end{definition}
We emphasize that all considered GENEOs to define crossed pairs of translations  must be in $\Cat{T}$.
The possibility of choosing $\Cat{T}$ in different ways reflects the various approaches an observer can use to judge the similarity between data. 

\begin{example}\label{ex:findelity}
Consider the smallest possible $\Cat{T}$ (that is, no arrows between different perception spaces and only the identity between equal spaces) representing an observer who cannot translate the data. In this case, $h_{\mb{O}}\Big((f_\alpha,t_\alpha),(f_\beta,t_\beta)\Big)=\infty$ whenever $(f_\alpha,t_\alpha)$ and $(f_\beta,t_\beta)$ act on different perception spaces, since there is no translation pair $\pi\colon (f_\alpha,t_\alpha)\leftrightharpoons_{\mathbf{T}}(f_\beta,t_\beta)$. Whenever the perception spaces are the same, there is only one translation pair, formed by two identity GENEOs. Thus the surrogate distance of $(f_\beta,t_\beta)$ from $(f_\alpha,t_\alpha)$ collapses to the cost of such translation pair, that is, 
\[
  \frac{1}{|X_\alpha|}\sum_{x \in X_\alpha}
d_{\mathrm{dt}}\Big( f_\beta(x),
f_\alpha(x)\Big)
\]
Note that whenever $d_{\mathrm{dt}}$ assigns $0$ to equal elements and $1$ to different ones, this coincides with the standard notion of \emph{fidelity} \cite{NautaTPNPSSKS23}.
\end{example}

\begin{theorem}
\label{prop:semi-metric}
    The function $h_\mb{O}$ is a hemi-metric on $\mathcal{G}$.
\end{theorem}

%
Notice that while $h_\mb{O}$ is a hemi-metric, one can easily get a pseudo-metric by making it symmetric: $d_\mb{O}:=\max\left(h_\mb{O}\Big((f_\alpha,t_\alpha),(f_\beta,t_\beta)\Big),h_\mb{O}\Big((f_\beta,t_\beta),(f_\alpha,t_\alpha)\Big)\right)$.
We choose to stay with the non-symmetric distance $h_\mb{O}$ since it should measures how far the observer $\mb{O}$ perceives the surrogate $(f_\beta,t_\beta)$ from the GEO to interpret $(f_\alpha,t_\alpha)$. We believe that for this kind of measurement, it is more natural to drop symmetry, like, for instance, in the case of fidelity (\Cref{ex:findelity}).


\subsection{Measures of Complexity}
\label{subsec:complexity}
In \Cref{subsec:cat} we have introduced string diagrams allowing for combining several building blocks taken from a given set of symbols $\Gamma$ and we have illustrated how the semantics assigns to each diagram a GEO. 
Here, we establish a way to measure the \textit{comfort} that an observer $\mathbb{B}$ has in dealing with a certain diagram. We call such measure the \emph{complexity} of a diagram relative to $\mathbb{O}$. 

To give a complexity to each diagram, we exploit the complexity assignment $\mc{C} \colon \Gamma \to \mathbb{R}^+$ of the observer $\mathbb{O}$ that provides a complexity to each building block.

\begin{definition}
\label{def:complexity}
Let $c$ be a diagram in $\Cat{Diag}_\Gamma$.
The \emph{complexity of a diagram} $c$ (relative to the observer $\mathbb{O}$), written $\com{c}$, is inductively  as follows:
\[
\begin{array}{rclrclrcl}
\com{\nbox{g}[][][][]} & := & \mathcal{C}(g) \;&\;
\com{\emptyCirc{}} & := & 0 \;&\;
\com{c_1\otimes c_2} & := & \com{c_1} + \com{c_2} \\
\com{\copyCirc[A]} & := & 0 \;&\;
\com{\symmCirc[A][B]} & := & 0 \;&\;
\com{c_1\circ c_2} & := & \com{c_1} + \com{c_2} \\
\com{\idCirc[A]} & := & 0 \;&\;
\com{\discardCirc[A]} & := & 0 
\end{array}
\]
\end{definition}
Shortly, the complexity of a diagram $c$ is the sum of all the complexities of the basic blocks occurring in $c$. 


\begin{example}[Number of Parameters]
The set of basic blocks $\Gamma$ may contain several generators that depend on one or more parameters whose value is usually learned during the training process. 
A common way to measure the complexity of a model is simply by counting the number of its parameter. This can be easily accommodated in our theory by fixing the function $\mathcal{C} \colon \Gamma \to \mathbb{R}^+$ to be the one mapping each generator $g\in \Gamma$ into its number of parameters. It is thus trivial to see that for all circuit $c$, $\com{c}$ is exactly the total number of parameters of $c$.
\end{example}

\begin{example}[Number of Nonlinearities]
Let us assume that $\Gamma$ contains  as building blocks the functions computing the linear combinations of $n$ given inputs, for every $n\in\N$ and for each tuple of real valued coefficients. Moreover, $\Gamma$ contains as building blocks some classic activation functions in machine learning, such as the Sigmoid and the ReLu activation function. For instance, in our theory an observer may define the complexity 
$\mathcal{C} \colon \Gamma \to \mathbb{R}^+$ to assign to each linear function the complexity of 0 and to each nonlinear function the complexity of 1. Then the complexity of each circuit $c$, $\com{c}$ is exactly the number of nonlinear functions applied in the circuit, e.g. the number of neurons in a multi-layer perceptrons with ReLu activation functions in the hidden layers and Sigmoid activation function in the output layer.
\end{example}

We notice that we defined the complexity function on syntactic diagrams and not on semantic objects. Indeed, an operator, like e.g. a GEO, can be realized by possibly several different diagrams, however the complexity of the different diagrams should be different. To understand this choice, imagine one has to define the complexity of a function that, given a certain array of integers, returns the array in ascending order. Clearly the complexity of this function should depend on the specific algorithm that is used to produce the output given a certain input, and not on the function itself.

\section{Learning and Explaining via GE(NE)Os Diagrams}
\label{sec:framework}

\Cref{sec:diagrams} introduces the basic definitions that can be operatively used to instantiate our framework. Indeed, \Cref{eq:semi-metric} defines a hemi-metric that can be used as a loss function to train a surrogate GEO to approximate another GEO, whereas \Cref{def:complexity} establishes a way to measure their interpretability in terms of elementary blocks. This section first shows how the learning of surrogate models is defined (\Cref{subsec:learning}), and then how we can easily extract explanations from the learned surrogate models (\Cref{subsec:int}). For the following we assume to have fixed an observer $\mathbb{O}=(\Cat{T},\mc{C})$ interested in a set of GEOs $\mathcal{G}$.

\subsection{Learning via GENEOs' Diagrams}
\label{subsec:learning}

Given two GEOs $\alpha,\beta\in\mc{G}$, with 
$\alpha=(f_\alpha,t_\alpha):(X_\alpha,G_\alpha)\to (Y_\alpha,K_\alpha)$
and $\beta=(f_\beta,t_\beta):(X_\beta,G_\beta)\to (Y_\beta,K_\beta)$, and the category $\Cat{T}$ of translation GENEOs, 
the hemi-metric $h_{\mb{O}}$ as defined in \Cref{eq:loss} expresses the cost of approximating $\alpha$ with $\beta$ via the available translation pairs, as illustrated in \Cref{fig:dist}. 
In order to apply our  framework to the problem of learning interpretable surrogate functions of a certain model on a certain dataset, from now on we assume that $\alpha$ is given, $\beta$ is learnable by depending on a set of parameters $\theta\in\mathbb{R}^n$, and $X_{dt}$ denotes the training set collecting the available input data. 
Therefore, learning $f_\beta$ can be cast as the problem of finding the parameters $\theta$, such that $h_{\mb{O}}(\alpha,\beta)$ is minimized on $X_{dt}$, i.e. that provide the lowest $cost(\pi)$ amongst  the $\pi=\Big((l^\pi_{\alpha,\beta},p^\pi_{\alpha,\beta}),(m^\pi_{\beta,\alpha},q^\pi_{\beta,\alpha})\Big)\colon \alpha\leftrightharpoons_{\mathbf{T}}\beta$:
\begin{equation}
\label{def:lea}
    \theta^*
    =\argmin_{\theta} \left(\inf_{\pi} \frac{1}{|X_{dt}|}\sum_{x\in X_{dt}} d_{\mathrm{dt}}\Big(m^\pi_{\beta,\alpha}(f_\beta(l^\pi_{\alpha,\beta}(x);\theta)), f_\alpha(x)\Big)\right)\ .
\end{equation}
From our definition, the two perception spaces  may be different. However, most frequently when learning surrogate functions, we have $W_\alpha = W_\beta = W$, for $W\in\{X,Y,\B G,\B K\}$, and there is only the translation pair $\pi=\big((id_{X},id_G)(id_{Y},id_K)\big)$. Thus, \Cref{def:lea} simplifies in $\argmin_{\theta}  \frac{1}{|X_{dt}|}\sum_{x\in X_{dt}} d_{\mathrm{dt}}\Big(f_\beta(x;\theta), f_\alpha(x)\Big)$,
which corresponds to the fidelity measure between $f_\alpha$ and $f_\beta$, commonly used in XAI.

\begin{example}[Classifier Explanations]
\label{ex:class}
Consider a classifier $f_\alpha$ equivariant w.r.t. the groups $\B G_\alpha$ and $\B K_\alpha=\B 1$, being $\B 1$ the trivial group. As an example, \Cref{fig:classification_diagram} illustrates two different GEOs $f_\beta$ and $f_\gamma$ that can be used to explain $f_\alpha$. Notice that if the observer $\mb{O}$ has no access to $f_\alpha$, i.e. $\mb{O}$ does not know how $f_\alpha$ is built (i.e. $f_\alpha$ is a black-box for $\mb{O}$), then $f_\alpha$ should be an atomic block in $\Gamma$. In this case, the observer $\mb{O}$ assigns to $f_\alpha$ the complexity $\mathcal{C}(f_\alpha)=\infty$.
\end{example}

\begin{example}[Supervised Learning]
Wether $f_\alpha$ denotes the function associating to each training input its label (i.e. the \textit{supervisor}), then $f_\beta$ and $f_\gamma$ from \Cref{fig:classification_diagram} are simply two models trained via supervised learning, and their distance to $f_\alpha$ is the accuracy (that can be thought of as the fidelity w.r.t. the ground-truth). 
\end{example}
$f_\beta$ and $f_\gamma$ differ in \Cref{ex:class} only from the fact that $f_\beta$ is equivariant on the same group $\B G_\alpha$ than $f_\alpha$, whereas $f_\gamma$ might not. In fact, in case $f_\gamma$ is not equivariant on $\B G_\alpha$ we may prove that $f_\gamma$ will be surely a non-optimal approximation.

\begin{proposition}
\label{th:upper_bound}
Let $\mathbf{T}$, $(f_\alpha,t_\alpha)$, $(f_\beta,t_\beta)$ as in Example \ref{ex:findelity} and let $NE$ be the set $\{(g,x)\in G_\alpha\times X \mid f_\beta(x) \neq f_\beta(g*x)\}$, i.e., the set containing all those couples falsifying equivariance of $f_\beta$ w.r.t. $G_\alpha$. Then
\[h_{\mathbb{O}}((f_\alpha,t_\alpha),(f_\beta,t_\beta)) \geq \frac{|NE|}{2\cdot |G_\alpha|}\]
\end{proposition}

\begin{remark}
    As stated in the introduction, single-hidden-layer neural networks are universal approximators but may require a large number of hidden neurons, increasing complexity. If we cap the model's complexity, a neural network may not always approximate a given model accurately. \Cref{th:upper_bound} further establishes a fidelity lower bound based on non-equivariant datapoints.
\end{remark}


\subsection{Suitable Surrogate GEOs}
\label{subsec:int}


We say that a GEO $(f_\alpha,t_\alpha)$ is \emph{explained} by another GEO $(f_\beta,t_\beta)$ at the level $\varepsilon$ for an observer $\mathbb{O}=(\Cat{T},\mc{C})$  if:
\[1.h_{\mb{O}}\Big((f_\alpha,t_\alpha),(f_\beta,t_\beta)\Big)\le\varepsilon;\qquad
2.\com{(f_\beta,t_\beta)} \le \com{(f_\alpha,t_\alpha)}.
\]
The second condition means that the complexity of the surrogate explaining model $(f_\beta,t_\beta)$ should be lower than the complexity of the given model $(f_\alpha,t_\alpha)$. 
While not guaranteed, this requisite can be ensured by designing $f_\beta$ with a suitable strategy. Recall that a model's complexity is defined by atomic building blocks in $\Gamma$, which are combined to form the model. Using the simplest possible blocks helps limit complexity, though their selection depends on the observer's knowledge and interpretability.
Moreover, different studies \cite{BeFrGiQu19} have shown how a proper domain-informed selection of GE(NE)Os, may strongly decrease the number of parameters necessary to solve a certain task w.r.t. a standard neural networks (as also shown by our experiments cf. \Cref{tab:params}).

\begin{example}
Given a set of GEOs $(f_i,t_i)\in \Gamma$, with complexity $k_i=\mc{C}((f_i,t_i))$, we can define $f_\beta$ as a linear combination of $(f_1,t_1), \ldots,(f_n,t_n)$. According to \Cref{def:complexity}, the complexity  $\com{f_\beta}$ would be $k_1+\ldots+k_n$, plus eventually the complexities of the scalar multiplications.
\end{example}

\section{Experiments}
\label{sec:exp}

In order to validate experimentally our theory, we build a classification task on MNIST dataset and rely on our framework to appropriately define an interpretable surrogate model. With our experiments we aim to answer two main research questions: wether personalized complexity measures are able to properly formalize an observer subjectivity, and if knowledge of the domain and of the complexity measured by an observer can lead to ad-hoc surrogate models with a better trade-off between complexity and accuracy. Thus for all the reported results, we assume to have fixed one (or more) given observers.


\subsection{Data}
The MNIST dataset contains $70,000$ grayscale (values from $0$ to $255$) images of handwritten digits (0-9), each image being $28 \times 28$ pixels. We linearly rescale the images so that the values lay in $[0,1]$. The images rescaled belong to $\{0,\frac{1}{255},\dots,1\}^{28 \times 28}$
We split our dataset into three stratified random disjoint subsets: training, validation, and test set, of $60\%$, $20\%$ and 20\% of images respectively.

\subsection{Models}

As opaque model, we employ a standard CNN, with the Tiny-Vgg architecture, that is composed by two convolutional layers as tail and a linear classifier head.
To realize our GEOs surrogate approximation, we use two different architectures.
From the MNIST training set, we extract randomly a set of patterns $p_i$.
These patterns are square cutouts of train images, with height ($H$) and width ($W$) of choice and with a center point chosen with probability proportional to the intensity of the image $x$:
\begin{align*}
p_i = x|_{Q_i}, \qquad & (c_{x_i}, c_{y_i}) \sim x \\
Q_i =\{c_{x_i} -\frac{W}{2},\dots,c_{x_i} +\frac{W}{2}\} &\times \{c_{y_i} -\frac{H}{2},\dots,c_{y_i} +\frac{H}{2}\}
\end{align*}

For each image $x$ we identify the presence of a pattern $p_i$ in position $(i,j)$ with the following function:
\begin{align*}
&f(x)_{p_i}:\{0,\frac{1}{255},\dots,1\}^{28 \times 28} \rightarrow \{0,\frac{1}{199920},\dots,1\}^{28 \times 28}  \\
&f_{p_i}(x)_{n,m} = 1- \frac{\sum_{(i,j) \in Q_i}\left|x((i,j)+(n,m))-p_i((i,j))\right|}{\textup{vol } Q_i}
\end{align*}

The choice of these specific patterns can be motivated by a domain knowledge or by the preferences that an observer can inject through a thoughtful design of theirs GEOs' building blocks for the classification task. 

The first GEO then performs a Image-Wide-Maxpool to create a flat vector with as many entries as are the patterns, and whose $i^{\textup{th}}$ entry indicates the intensity with which the pattern was identified within the image
\[
L_i = \textup{max}_{n,m}(f_{p_i}(x)_{n,m}) 
\]

These intensities are then linearly combined with an activation function to identify the correct digit
\[
\textup{OUT}^k = \sigma \left( \sum_j \gamma_{j}^k L_j +b^k \right)
\]



The second GEO instead, after the identification of patterns, selects for each pattern the position with the maximum activation through the Channel-Wise-Max ($CWM$)
\[
CWM(f_{p_i}(x))_{n,m}=\begin{cases}
    s &\textup{ if } s=\textup{max}(f_{p_i}(x))\\ 0 & \textup{otherwise}
\end{cases}
\]
These matrices of activations are then linearly combined with a downstream nonlinear activation function
\[
L_{n,m} = \sigma\left(\left(\sum_{i} w_{i} \cdot CWM(f_{p_i}(x))_{n,m}\right) + b_i\right)
\]
The entries of this matrix are then linearly combined with a final sigmoidal activation function to produce the output of the model
\[
\textup{OUT}^k = \sigma \left( \left( \sum_{ij} w_{ji}^k \cdot L_{ji}\right) +b^k \right)
\]


To compare results, we chose a series of simple Multi-Layer Perceptrons, trained directly on the MNIST dataset. In particular, we used MLPs with the following configurations: with no hidden layers, with one hidden layer of dimension $5$, $7$, $20$ and $40$.
The two models with hidden layers of dimension $5$ and $7$ are chosen to create MLPs with number of parameters similar to our GEOs. In \Cref{tab:params} we report the most relevant characteristics of all the models we compare in our experiments. 

\subsection{Experiment Setup}

 
 



\begin{table}[t]
    \centering
     \scriptsize
    \begin{tabular}{ c c c c}
 Model & Params & Epochs & LR  \\ \hline
 \rule{0pt}{2.5ex}CNN & $228010$ & $3$ & $3e-3$   \\ 
 \rule{0pt}{2.5ex}MLP & $31810$ & $57$ & $2e-4 $\\
 MLP & $15910$ & $57$ & $1e-4 $ \\
 MLP & $7850$  &$5$   & $2e-3$ \\
 MLP & $5575$ & $58$ & $2e-4$ \\
 MLP & $3985$ & $58$ & $2e-4$ \\
 MLP & $3190$ & $9$ & $2e-3$ \\
 &&&\\
 &&&\\
 &&&\\
 &&&
 \end{tabular}
 \begin{tabular}{ c c c c c}
  Model & Params & Epochs & LR &PATTERNS \\ \hline
 \rule{0pt}{2.5ex}GEO$_{1}$ & $5010$ & $296$& $3e-3$ & 500\\
 GEO$_{1}$ & $3510$ & $148$  & $7e-3$ & 350\\
 GEO$_{1}$ & $1710$ & $456$ & $2e-2$ & 170\\
 GEO$_{1}$ & $1510$ & $564$ & $1e-2$ & 150\\
 GEO$_{1}$ & $1210$ & $496$ & $2e-2$ & 120\\
 GEO$_{1}$ & $990$ & $198$ & $5e-2$ & 98\\

 \rule{0pt}{2.5ex}GEO$_{2}$ & $8101$ & $39$   & $1e-3$ & 250  \\
 GEO$_{2}$ & $8051$ & $496$   & $1e-3$ & 200  \\
 GEO$_{2}$ & $8001$ & $483$   & $1e-3$ & 150  \\
 GEO$_{2}$ & $7951$ & $335$   & $1e-3$ & 100  \\
 GEO$_{2}$ & $7901$ & $451$   & $1e-3$ & 50  \\

\end{tabular}
    \caption{The different models utilized with the relative hyperparameters, chosen on the validation set.}
    \label{tab:params}
\end{table}

We performed the experiments training all models over the ground truth.

We employed early stopping on the validation set to determine the optimal number of training epochs. 
The accuracy was then evaluated on a separate test set.
We also trained a portion of our models on a rescaled version of MNIST for which every separate group of $2 \times 2$ points was substituted with the max of the four pixels, effectively reshaping the images to $14 \times 14$ and allowing us to compare also models which start from different perception spaces.

\subsection{Results}
We first follow our theoretical framework to define the translation diagram of our experimental setup. Indeed, we are in a classical classification scenario, that can be easily represented by the graph in \Cref{fig:classification_diagram}.

\begin{figure}
    \centering
    \includegraphics[width=0.5\linewidth]{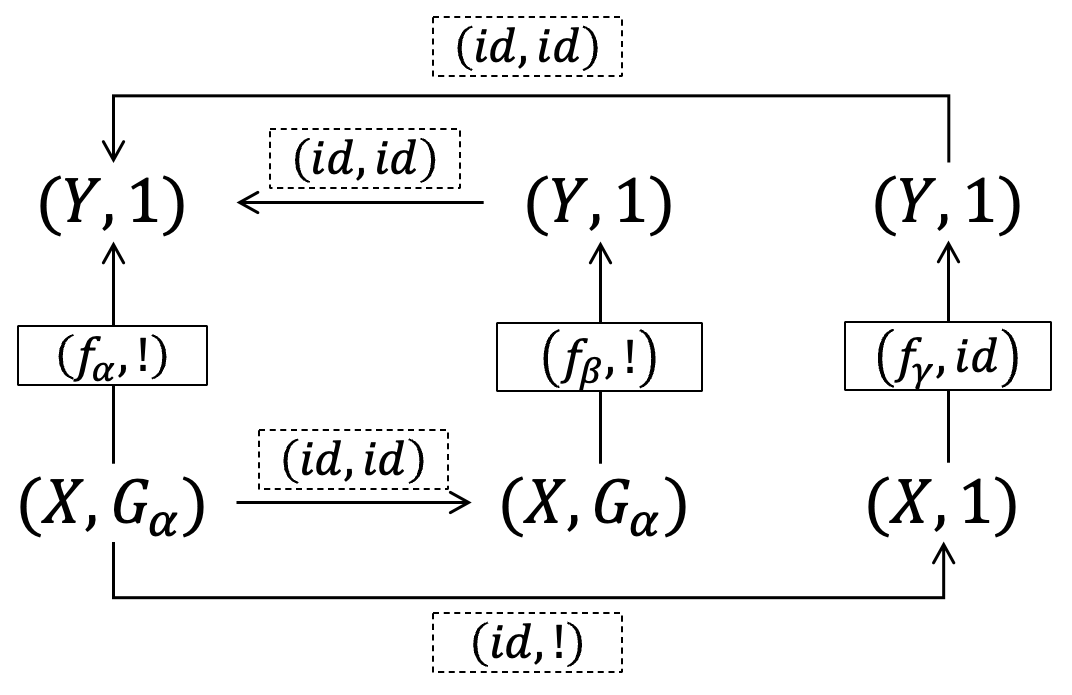}
    \caption{Diagrams of two GEOs explaining a given GEO, where $!$ represents the annihilator homomorphism from any group to the trivial group $\B 1$.}
    \label{fig:classification_diagram}
\end{figure}

We start from the basic perception space $(X,G_\alpha)$ that is, our image dataset $X$ and the group of admissible transformations $G_\alpha$. Here, we have translations as admissible group actions in $G_\alpha$ and $f_\alpha$ is the opaque CNN. Our first GEO model $f_\beta$ operates on the same perception space $(X,G_\alpha)$, as it works on the thorus of the images, preserving translations. Therefore, the translation GENEO is composed given by the couple $(id_X,id_{G_\alpha})$. Both the second GEO and the MLP, represented by $f_\gamma$ instead do not preserve any transformation in the group. Therefore the perception space becomes $(X,\B 1)$ where $\B 1$ denotes the trivial group. Being $!$ the annihilator homomorphism from any group to the trivial group, the translation GENEO for this GEO is given by the couple $(id_X,!)$. All models have $(Y,1)$ as their output, since they all work on the space of output classes. 

To show how the subjectivity of an observer may influence the results in practice, we measure complexity using two measures:
Firstly we assign complexity $1$ to each parameter of the model and we sum over all the parameters. Then we assign complexity $1$ to all the non-linearities of the model, summing over all the non linearities.
We report the performances obtained by the different models in \Cref{tab:results1} and we also compare the results with a different perception space in \Cref{tab:results2} where we present the results for resized images.
\begin{table}[t]
\begin{center}
\scriptsize
\begin{tabular}{ c  c c c c }
 Model & $\mathcal{C}_1$ & $\mathcal{C}_2$ & Acc & Fid \\ \hline
   \rule{0pt}{2.5ex}CNN & $228010$ & 37578 & $97.8 \%$ & \\ 
  \rule{0pt}{2.5ex}MLP & $31810$ & 50 &$96.3 \%$ &$93.6\%$\\ 
 MLP & $15910$ & 30&$94.1 \%$ &$93.5\%$\\
 MLP & $7850$ & 10 &$91.8 \%$ &$90.9\%$\\
 MLP & $5575$ & 17&$90.3 \%$ &$89.6\%$\\
 MLP & $3985$ & 15&$85.4 \%$ &$86.1\%$\\
 MLP & $3190$ & 14&$85.1 \%$ &$80.3\%$\\
 \end{tabular}
 \begin{tabular}{ c c c c c }
  Model & $\mathcal{C}_1$ & $\mathcal{C}_2$ & Acc &Fid\\ \hline
   \rule{0pt}{2.5ex}GEO$_{1}$ & $5010$ & 510 &$96.6 \%$ &$92.5\%$ \\
 GEO$_{1}$ & $3510$ & 360 &$95.4 \%$ &$91.9\%$\\
 GEO$_{1}$ & $1710$ & 180 &$95.3 \%$ &$92.4\%$\\
 GEO$_{1}$ & $1510$ & 160 & $93.7 \%$ &$90.7\%$\\
 GEO$_{1}$ & $1210$ & 130&$93.4\%$ &$91.5\%$\\
 GEO$_{1}$ & $990$ & 100 &$92.2 \%$ &$89.3\%$\\ 
  \rule{0pt}{2.5ex}GEO$_{2}$ & $8101$ & $511$   & $92.9\%$ &$92.5\%$ \\
 GEO$_{2}$ & $8051$ & $411$   & $92.0\%$ &$91.8\%$ \\
 GEO$_{2}$ & $8001$ & $311$   & $92.6\%$ &$91.6\%$ \\
 GEO$_{2}$ & $7951$ & $211$   & $91.3\%$ &$91.1\%$ \\
 GEO$_{2}$ & $7901$ & $111$   & $88.5\%$ &$91.4\%$ 
\end{tabular}
\caption{Models with relative complexities, accuracies and fidelities w.r.t CNN.}
\label{tab:results1}
\end{center}
\end{table}
\begin{table}[t]
\begin{center}
\scriptsize
\begin{tabular}{ c  c c c c }
 Model & $\mathcal{C}_1$ & $\mathcal{C}_2$ & Acc \\ \hline
\rule{0pt}{2.5ex}MLP & $8290$ & $50$ &$96.3 \%$ \\ 
 MLP & $1970$ & $10$ &$91.8 \%$ \\
 MLP & $1459$ & $17$ &$90.3 \%$ \\
 MLP & $1045$ & $15$& $86.3 \%$ \\
 \end{tabular}
 \begin{tabular}{ c c c c c }
  Model & $\mathcal{C}_1$ & $\mathcal{C}_2$ & Acc \\ \hline
   \rule{0pt}{2.5ex}GEO$_{1}$ & $5010$ & $510$ &$95.9 \%$  \\
 GEO$_{1}$ & $3510$ & $360$ & $95.5 \%$ \\
 GEO$_{1}$ & $1710$ & $180$ & $93.6 \%$ \\
 GEO$_{1}$ & $990$ & $100$ &$91.1 \%$ \\ 
  \rule{0pt}{2.5ex}GEO$_{2}$ & $2221$ & $511$   & $93.1\%$  \\
 GEO$_{2}$ & $2121$ & $411$   & $92.5\%$  \\
 GEO$_{2}$ & $2021$ & $111$   & $89.5\%$  \\
\end{tabular}
\caption{The output of some of the models trained on a rescaled version of the starting perception space. The hyperparameters have been kept the same as the non rescaled experiments}
\label{tab:results2}
\end{center}
\end{table}
\begin{figure}[h!]
    \begin{subfigure}[t]{0.45\linewidth}
    \includegraphics[width=\linewidth]{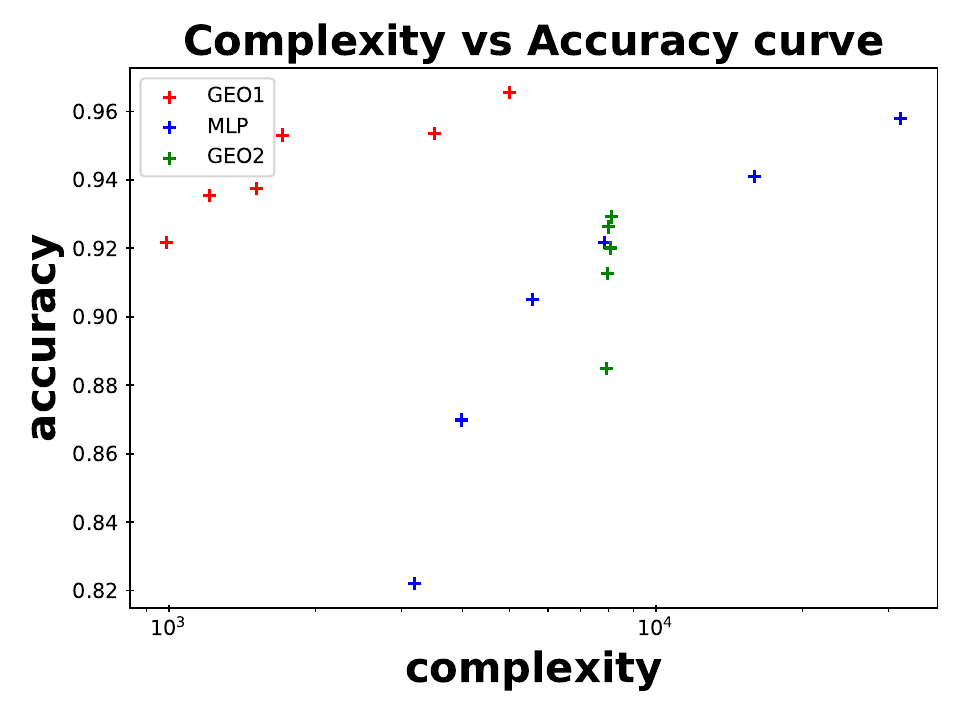}
    \caption{ The proposed GEOs can outperform, at similar complexity, model-agnostic MLPs. Notice that the translational equivariant GEO is able to perform much better at minor complexity.}
    \end{subfigure} \hfill
    \begin{subfigure}[t]{0.45\linewidth}    
    \includegraphics[width=\linewidth]{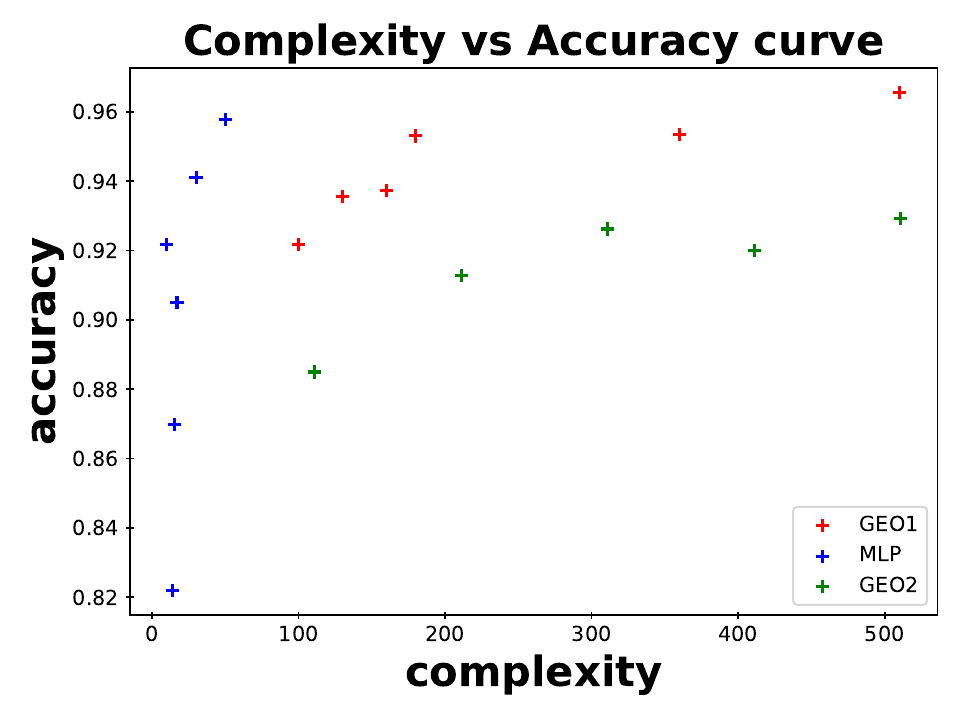}
    \caption{A second observer, who ascribes complexity only to the non-linearities, can have a different complexity vs accuracy curve.}
    \end{subfigure}
    \begin{subfigure}[t]{0.45\linewidth}
    \includegraphics[width=\linewidth]{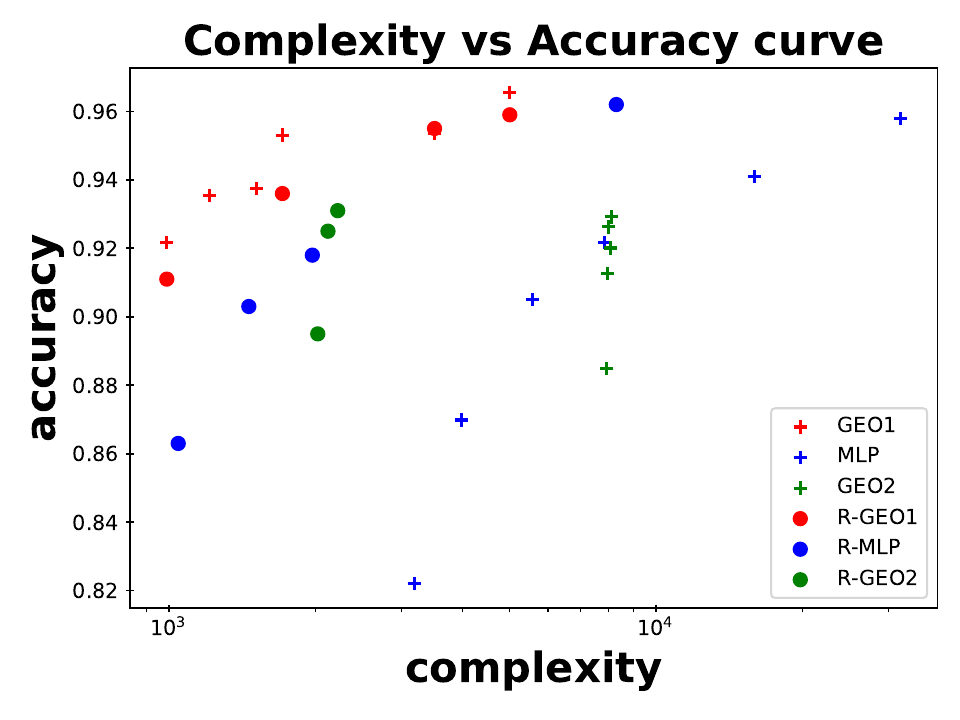}
    \caption{Changing the starting perception space does not affect significantly  the performances, whereas the first observer sees GEO$_1$ with unchanging complexities and GEO$_2$ and the MLP with much smaller complexity.}
    \end{subfigure} \hfill        
    \begin{subfigure}[t]{0.45\linewidth}    
    \includegraphics[width=\linewidth]{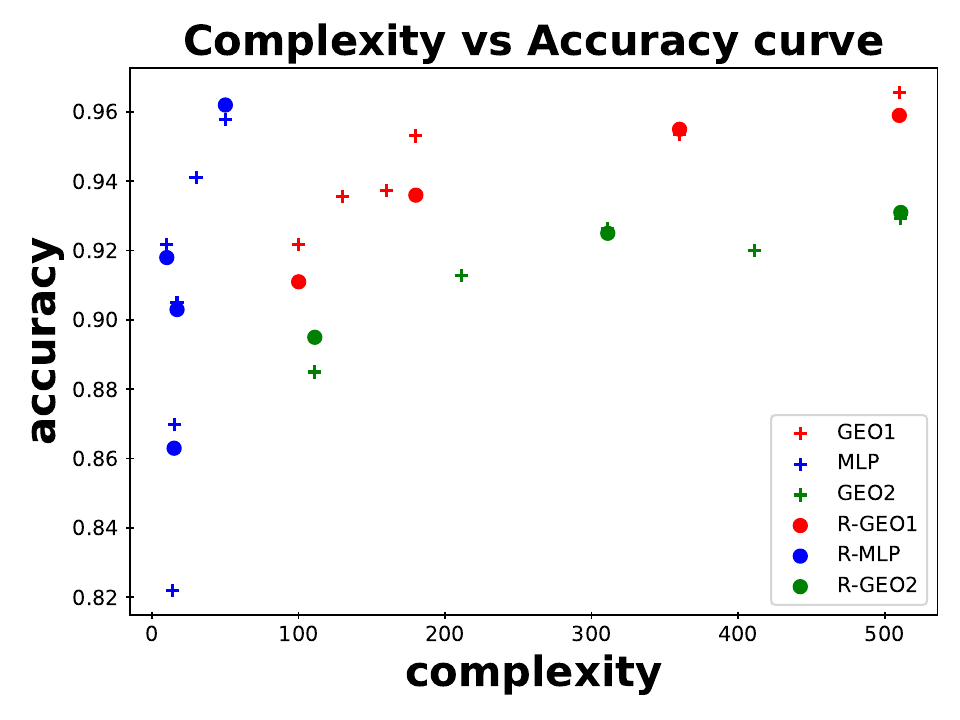}
    \caption{The different observer is not affected by any change in the measured complexity.}
    \end{subfigure}
    \caption{Accuracy vs complexity comparisons.}
    \label{fig:acc-comp}
\end{figure}
The results show that the models built via thoughtful GEOs' building blocks can approximate quite well the original task, providing models that are less complex for both the measure specified by the observer. The complexity vs accuracy curves reprensenting the experiments are shown in \Cref{fig:acc-comp}.

\section{Related Work}
\label{sec:relwork}

Explainable AI has become a fundamental field in AI that covers  methodologies designed to provide understandable explanations of the inner workings of a ML model to a human being \cite{Kay23}. Roughly, XAI methods can be categorized into post-hoc methods, i.e., methods aiming to explain another trained opaque ML models, and interpretable-by-design methods, i.e., ML models that provide explanations to the users inherently, by virtue of their intrinsic transparency \cite{BodriaGGNPR23,YangWWCHLLYWGAK23}. One of the most well-known techniques for post-hoc explnations is to train a surrogate interpretable model to reproduce the same output as an opaque model \cite{LualdiSS22,HeidariTR23,CollarisGJWP23}. In this regard, our paper provides a solid mathematical framework that subsumes both these two paradigms in the same theory.

A key point in XAI is the way the quality of the provided explanations can be measured. For instance, explanations and interpretability can be evaluated qualitatively (user studies) or quantitatively (direct model metrics) \cite{AlangariMMA23a,NautaTPNPSSKS23,MirzaeiMAW24,ZhukovBG23}. Qualitative measures include user performance, engagement, and explanation clarity \cite{PaniguttiBFGPPR23,ColleyKHV23,Schulze-Weddige21,AroraPSCLN22}. Quantitative measures include explanation completeness \cite{WagnerKGHWB19}, fidelity \cite{GuidottiMRTGP19}, classification accuracy \cite{HarderBP20}, and faithfulness \cite{MurdochSLAY19}. Complexity measure of explanations is often used for logic-based explainers \cite{ciravegna2023logic}, but it is generally limited to be a count on the number of propositional variables in a formula. While this can easily be accomodated in our framework, up to the author knowledge, no other methods consider complexity measures from the perspective of an observer, offering flexibility in choosing suitable metrics for the task and models.


While there is a large agreement on the needs for XAI models, there are very few works that try to provide a formal mathematical theory of explanations and/or interpretability for ML models. 
For instance, in \cite{tull2024towards} the authors propose a new class of ``compositionally-interpretable'' models, which extend beyond intrinsically interpretable models to include causal models, conceptual space models, and more, by using category theory. \cite{giannini2024categorical} proposes a framework based on Category Theory and Institution Theory to define explanations and (explainable) learning agents mathematically. However, these works do not provide a practical measure for the interpretability of the models, completely omit the formalization of an observer, and do not take into account the notion of group equivariant operators. Another seminal work is \cite{HoffmanK17,HoffmanMK17}, which provides a more general foundation framework based on properties and desiderata for interpretable ML. However, it does not make any specific mention to a proper mathematical framework.

Finally, our framework is based on the theory of GE(NE)Os, which has been already used to bridge Topological Data Analysis (TDA) and ML. For instance, GENEOs originates from persistent homology with G-invariant non-expansive operators and have been succesfully applied for 1D-signal comparisons and image recognition based  on topological features \cite{FrJa16}. Moreover, GENEOs have been applied to protein pocket detection \cite{BeFrGiQu19,Boeal24} and graph comparison \cite{BoFeFr25}. While as observed in \cite{Boeal24} GENEOs are more inherently interpretable due to a limited dependency on parameters, the theory we present in this paper significantly extend the previous applications, by aiming at the formalization of a more sound XAI theory evaluable quantitatively and based on observers' preferences.

\section{Conclusions and Future Work}
\label{sec:conc}

This work explores the theoretical properties of GE(NE)Os to build a theoretical framework to build surrogate interpretable models, and measure in a rigorous way the trade-off between complexity and performance. By formally proving the properties of our framework and with the experiments that we provide, we lay the groundwork for future research and opening avenues for practical applications in analyzing and interpreting complex data transformations. Our proposal highlights how it is possible to frame the theory of interpretable models through GE(NE)Os and opens new interesting research directions for Explainable AI. One such direction will be to formally describe existing machine learning models in terms of GE(NE)Os, to study the best interpretable approximations for typical tasks. Moreover, an interesting possible research could be to realize interpretable latent space compression through the use of GE(NE)Os.


%
%
%
\bibliographystyle{splncs04}
\bibliography{refexplainability}

\appendix

\section{Additional Results and Proofs}
\label{app:proofs}
\begin{proposition}
\label{prop:topo}
Let $(X,d_X,{\B G}, *)$ be a perception pair. The followings hold.
\begin{itemize}
    \item[(a)] $(\B G,\circ)$ is a topological group.
    \item[(b)] The action of $\B G$ on $X$ is continuous.
\end{itemize}
\end{proposition}
\begin{proof}
To prove (a) it is sufficient to prove that the maps 
$(g',g'')\mapsto g'\circ g''$   
and $g\mapsto g^{-1}$
are continuous.
First of all, we have to prove that if a sequence $(g_i')$ converges to $g'$
and a sequence $(g_i'')$ converges to $g''$
in $G$, then the sequence $(g_i'\circ g_i'')$ converges to $g'\circ g''$
    in $G$. 
%
We observe that, for every $x\in X$,
{\small \begin{align*}
d_X((g_i'\circ g_i'')*x,(g'\circ g'')*x)&= d_X(g_i'*( g_i''*x),g'*( g''*x))\\
&\le d_X(g_i'*( g_i''*x),g_i'*( g''*x))\\
&+d_X(g_i'*( g''*x),g'*( g''*x))\\
&= d_X(g_i''*x,g''*x)\\
&+d_X(g_i'*( g''*x),g'*( g''*x))\\
&\le d_G(g_i'',g'')+d_G(g_i',g').
\end{align*}}
Thus,
$d_G(g_i'\circ g_i'',g'\circ g'')\le d_G(g_i'',g'')+d_G(g_i',g')$.
This proves the first property.
Then, we have to prove that if a sequence $(g_i)$ converges to $g$
in $G$, then the sequence $(g_i^{-1})$ converges to $g^{-1}$ in $G$. We have that
{\small
\begin{align*}
d_X(g_i^{-1}*x,g^{-1}*x)&= d_X(g_i*(g_i^{-1}*x),g_i*(g^{-1}*x))\\
&= d_X((g_i\circ g_i^{-1})*x,(g_i\circ g^{-1})*x)\\
&= d_X(x,(g_i\circ g^{-1})*x)\\
&= d_X((g\circ g^{-1})*x,(g_i\circ g^{-1})*x)\\
&= d_X(g*(g^{-1}*x),g_i*(g^{-1}*x))\\
&\le d_G(g,g_i).
\end{align*}}
Therefore,
$d_G(g_i^{-1},g^{-1})\le d_G(g,g_i)$.
This proves our second property.

Now we prove (b).
    We have to prove that if a sequence $(x_i)$ converges to $x$
    in $X$ and a sequence $(g_i)$ converges to $g$
    in $G$, then the sequence $(g_i*x_i)$ converges to $g*x$
    in $X$. Since $\lim_{i\to\infty} x_i=x$ and $\lim_{i\to\infty} g_i=g$, then 
    $\lim_{i\to\infty} d_X(x_i,x)=0$ and 
    $\lim_{i\to\infty} d_X(g_i*x,g*x)=0$. 
    We have that, for every $x\in X$,
    {\small \begin{align*}
d_X(g_i*x_i,g*x) &\le d_X(g_i*x_i,g_i*x)+d_X(g_i*x,g*x)\\
&= d_X(x_i,x)+d_X(g_i*x,g*x)\\
&\le d_X(x_i,x)+d_G(g_i,g).
\end{align*}}
\end{proof}


%
%
\paragraph{Semantics of diagrams.}
It is convenient to first fix some notation.
\begin{remark}[Notation]\label{rem:notation}
Given two sets $X$ and $Y$, we write $X\times Y$ for their Cartesian product 
and $\sigma_{X,Y}\colon X \times Y \to Y \times X$ for the symmetry function mapping $(x,y)\in X\times Y$ into $(y,x)\in Y\times X$; given two functions $f_1\colon X_1 \to Y_1$ and $f_2\colon X_2 \to Y_2$, we write $f_1\times f_2 \colon X_1\times X_2 \to Y_1 \times Y_2$ for the function mapping $(x_1,x_2)\in X_1\times X_2$ into $(f(x_1),f(x_2))\in Y_1\times Y_2$; Given $f\colon X \to Y$ and $g\colon Y \to Z$, we write $g\circ f \colon X \to Z$ for their composition. For an arbitrary set $X$, we write $\id_{X}\colon X \to X$ for the identity function, and $\Delta_X \colon X \to X \times X$ for the copier function mapping $x\in X$ into $(x,x)\in X \times X$; We write $1$ for a singleton set that we fix to be $\{\star \}$ and $!_X\colon X \to 1$ for the function mapping any $x\in X$ into $\star$. 

Given two perception spaces $(X,G)$ and $(Y,K)$, their direct product written $(X,G) \otimes (Y,K)$ is the perception space $(X\times Y,G\times K)$, where the distance on $X\times Y$ is defined as $d_{X\times Y} ((x_1,y_1) \, , \, (x_2,y_2)):=\max\{d_X(x_1,x_2) \, , \, d_{Y}(y_1,y_2) \}$ while the group action is defined pointwise, that is $(g,k)*(x, y)=(g*x, k*y)$. We write $\sigma_{(X,G) , (Y,K)} \colon (X,G) \otimes (Y,K) \to  (Y,K) \otimes (X,G)$ as $(\sigma_{X, Y}, \sigma_{G,K})$.
\end{remark}

With this notation one can extend the above structures of sets and functions to perception spaces and GEOs as illustrated in Table \ref{TableGeoAlgebra}. By simply checking that the definitions in Table \ref{TableGeoAlgebra} provide GEOs, one can prove the following result.

\begin{table}[t]
\[\arraycolsep=1pt\begin{array}{rcl}
\id_{X,G} &\defeq& (\id_X, \id_G) \colon (X,G) \to (X,G)\\
\Delta_{X, G} &\defeq& (\Delta_{X},\Delta_G) \colon (X,G) \to (X,G)\otimes (X,G)\\
!_{X,G} &\defeq& (!_{X},!_G) \colon (X,G) \to (1,1)\\
\sigma_{(X,G) , (Y,K)} &\defeq& (\sigma_{X, Y}, \sigma_{G,K}) \colon (X,G) \otimes (Y,K) \to  (Y,K) \otimes (X,G)\\
(f,t)\circ (f',t') &\defeq& (f'\circ f,t'\circ t)\colon (X,G) \to (Z,L)\\
(f_1,t_1) \otimes (f_2,t_2) &\defeq& (f_1\times f_2, t_1\times t_2)\colon (X_1,G_1)\otimes (X_2,G_2)  \to (Y_1,K_1) \otimes (Y_2,K_2)
\end{array}
\]
\caption{The CD category of GEOs. Above $(f,t)\colon (X,G) \to (Y,K)$, $(f',t')\colon (Y,K) \to (Z,L)$ and $(f_1,t_1)\colon (X_1,G_1) \to (Y_1,K_1)$, $(f_2,t_2)\colon (X_2,G_2) \to (Y_2,K_2)$ are GEOs. The notation on the right hand side is in \Cref{rem:notation}}\label{TableGeoAlgebra}.
\end{table}

\begin{lemma}$\Cat{GEO}$ is a CD category in the sense of \cite{ChoJ19}. 
\end{lemma}

From this fact, and the observation that $\Cat{Diag}_\Gamma$ is the (strict)  CD category freely generated from the monoidal signature $\Gamma$, one obtains that, for each interpretation $\mathcal{I}$, there exists a unique CD functor $\sem{-}\colon {\bf Diag} \to {\bf GEO}$ extending $\mathcal{I}$. Its inductive definition is illustrated in \Cref{tab:semantics}

\begin{table}[t]
\[
\begin{array}{rclrclrcl}
\sem{\nbox{g}[][][][]} & := & \mathcal{I}(g) \;&\;
\sem{\emptyCirc{}} & := & \id_{1,1} \;&\;
\sem{c_1\otimes c_2} & := & \sem{c_1} \otimes \sem{c_2} \\
\sem{\copyCirc[A]} & := & \Delta_{\mathcal{I_S}(A)} \;&\;
\sem{\symmCirc[A][B]} & := & \sigma_{\mathcal{I_S}(A),\mathcal{I_S}(B)} \;&\;
\sem{c_1\circ c_2} & := & \sem{c_1} \circ \sem{c_2} \\
\sem{\idCirc[A]} & := & id_{\mathcal{I_S}(A)} \;&\;
\sem{\discardCirc[A]} & := & !_{\mathcal{I_S}(A)} 
\end{array}
\]
\caption{The semantics $\sem{-}\colon {\bf Diag} \to {\bf GEO}$ for an interpretation $\mathcal{I}$. Operations and constants occurring on the right hand side of the above equations are those in \Cref{TableGeoAlgebra}. Above $\mathcal{I_S}$ is a function mapping each $A\in \mathcal{S}$ in a perception space such that, for all $g\in \Gamma$ with arity $A_1 \otimes \dots \otimes A_n$ and coarity $B_1 \otimes \dots \otimes B_m$, the source of $\mathcal{I}(g)$ is  $\bigotimes_{i=1}^n \mathcal{I_S}(A_i)$ and its target is  $\bigotimes_{j=1}^m\mathcal{I_S}(B_j)$.}\label{tab:semantics}
\end{table}


%
%
\paragraph{Cost of translation pairs for infinite perception spaces.}
Here we explain how the cost of translation pairs defined in \eqref{eq:loss} can be defined for arbitrary sets $X_\alpha$.

To proceed, we need to equip each metric space 
$X_\alpha$ with a Borel probability measure $\mu_\alpha$, in the spirit of \cite{CaFrQuSa23}. 
In simple terms, the measure $\mu_\alpha$ represents the probability of each data point in $X_\alpha$ appearing in our experiments.
We will assume that all GENEOs in $\mathbf{T}$ are not just distance-decreasing (i.e., non-expansive) but also  \emph{measure-decreasing}, i.e., if 
$(l_{\alpha,\beta},p_{\alpha,\beta}):(X_\alpha,G_\alpha)\to (X_\beta,G_\beta)$ belongs to $\mathbf{T}$ and the set $A\subseteq X_\alpha$ is  measurable for $\mu_\alpha$, then $l_{\alpha,\beta}(A)$ is measurable for $\mu_\beta$, and 
$\mu_\beta(l_{\alpha,\beta}(A))\le\mu_\alpha(A)$. Moreover, we assume that  the function $f_{\alpha,\beta}:X_\alpha\to \R$, defined  for every $x\in X_\alpha$ as
$f_{\alpha,\beta}(x):=d_{\mathrm{dt}}\Big((m_{\beta,\alpha}\circ f_\beta\circ l_{\alpha,\beta})(x),
f_\alpha(x)\Big)$, is integrable with respect to $\mu_\alpha$.
\begin{definition}
Let $\pi=\Big((l_{\alpha,\beta},p_{\alpha,\beta}),(m_{\beta,\alpha},q_{\beta,\alpha})\Big)$ be a crossed translation pair from $(f_\alpha,t_\alpha)\colon (X_\alpha,G_\alpha)\to (Y_\alpha,K_\alpha)$ to $(f_\beta,t_\beta)\colon (X_\beta,G_\beta)\to (Y_\beta,K_\beta)$.
The \emph{functional cost} of $\pi$, written $\cost(\pi)$, is defined as follows.
\begin{equation}
\label{eq:loss2}
    \cost(\pi) = \int_{X_\alpha}
d_{\mathrm{dt}}\Big((m_{\beta,\alpha}\circ f_\beta\circ l_{\alpha,\beta})(x),
f_\alpha(x)\Big)\ d\mu_\alpha.
\end{equation} 
\end{definition}

\paragraph{Proof of \Cref{prop:semi-metric}. } For  sake of generality, we illustrate the proof for the  case 
where $\cost(\pi)$ is defined as in \eqref{eq:loss2}. The case of $\cost(\pi)$ as in \eqref{eq:loss} follows by fixing 

\noindent\begin{minipage}{0.50\textwidth} $\mu_\alpha$ as  uniform Borel measure. Let us prove that $h_{\mb{O}}$ enjoys the triangle inequality, i.e., $h_{\mb{O}}(\alpha,\gamma)\leq h_{\mb{O}}(\alpha,\beta)+h_{\mb{O}}(\beta,\gamma)$, where $\alpha$, $\beta$ and $\gamma$ are three 
\end{minipage}
\begin{minipage}{0.50\textwidth}
\[
\begin{array}{rcl}
\alpha\defeq(f_\alpha,t_\alpha)\colon(X_\alpha,G_\alpha)\to (Y_\alpha,K_\alpha)\\
\qquad \beta \defeq (f_\beta,t_\beta)\colon(X_\beta,G_\beta)\to (Y_\beta,K_\beta)\\
\qquad \gamma\defeq (f_\gamma,t_\gamma)\colon(X_\gamma,G_\gamma)\to (Y_\gamma,K_\gamma)
\end{array}
\]\end{minipage}
GEOs in $\mathcal{G}$ illustrated  on the right. We consider three translation pairs:
\[\begin{array}{rcl}
\pi_1 &\defeq&\Big((l_{\alpha,\beta},p_{\alpha,\beta}),(m_{\beta,\alpha},q_{\beta,\alpha})\Big)\colon \alpha \leftrightharpoons_{\mathbf{T}} \beta\\
\pi_2&\defeq&\Big((l_{\beta,\gamma},p_{\beta,\gamma}),(m_{\gamma,\beta},q_{\gamma,\beta})\Big)\colon \beta \leftrightharpoons_{\mathbf{T}} \gamma\\
\pi_3&\defeq&\pi_2\circ \pi_1 = \Big((l_{\beta,\gamma}\circ l_{\alpha,\beta},p_{\beta,\gamma}\circ p_{\alpha,\beta}),(m_{\beta,\alpha}\circ m_{\gamma,\beta},q_{\beta,\alpha}\circ q_{\gamma,\beta})\Big)\colon \beta \leftrightharpoons_{\mathbf{T}} \gamma
\end{array}\]
Please note that if no crossed pair like $\pi_1$ or  $\pi_2$ exists, then $h_{\mb{O}}(\alpha,\beta)+h_{\mb{O}}(\beta,\gamma)=\infty$, and hence the triangle inequality trivially holds. By definition their costs are
\[\begin{array}{rcl}
\cost(\pi_1)& = & \int_{X_\alpha}
d_{\mathrm{dt}}\Big((m_{\beta,\alpha}\circ f_\beta\circ l_{\alpha,\beta})(x),
f_\alpha(x)\Big)\ d\mu_\alpha\\
\cost(\pi_2) & =& \int_{X_\beta}
d_{\mathrm{dt}}\Big((m_{\gamma,\beta}\circ f_\gamma\circ l_{\beta,\gamma})(y),
f_\beta(y)\Big)\ d\mu_\beta\\
\cost(\pi_3) & = &\int_{X_\alpha}
d_{\mathrm{dt}}\Big((m_{\beta,\alpha}\circ m_{\gamma,\beta}\circ f_\gamma\circ l_{\beta,\gamma}\circ l_{\alpha,\beta})(x),
f_\alpha(x)\Big)\ d\mu_\alpha
\end{array}
\]
\noindent Since $(m_{\beta,\alpha},q_{\beta,\alpha})$ is a GENEO, we have that for every $y\in X_\beta$,
\[d_{\mathrm{dt}}\Big((m_{\gamma,\beta}\circ f_\gamma\circ l_{\beta,\gamma})(y),
f_\beta(y)\Big)\ge d_{\mathrm{dt}}\Big((m_{\beta,\alpha}\circ m_{\gamma,\beta}\circ f_\gamma\circ l_{\beta,\gamma})(y),
(m_{\beta,\alpha}\circ f_\beta)(y)\Big)\]
and hence, setting $y:=l_{\alpha,\beta}(x)$
and recalling that $l_{\alpha,\beta}$ is measure-decreasing,
\begin{align*}
&\int_{X_\beta}d_{\mathrm{dt}}\Big((m_{\gamma,\beta}\circ f_\gamma\circ l_{\beta,\gamma})(y),
f_\beta(y)\Big)\ d\mu_\beta\\
&\ge \int_{X_\alpha} d_{\mathrm{dt}}\Big((m_{\beta,\alpha}\circ m_{\gamma,\beta}\circ f_\gamma\circ l_{\beta,\gamma}\circ l_{\alpha,\beta})(x),
(m_{\beta,\alpha}\circ f_\beta(y)\circ l_{\alpha,\beta})(x)\Big)\ d\mu_\alpha.
\end{align*}
Therefore, we have that $\cost(\pi_1)+\cost(\pi_2)=$
%
\[\begin{array}{ll}

=& \int_{X_\alpha} d_{\mathrm{dt}}\Big((m_{\beta,\alpha}\circ f_\beta\circ l_{\alpha,\beta})(x),
f_\alpha(x)\Big)\ d\mu_\alpha +\int_{X_\beta}d_{\mathrm{dt}}\Big((m_{\gamma,\beta}\circ f_\gamma\circ l_{\beta,\gamma})(y),
f_\beta(y)\Big)\ d\mu_\beta\\
\ge& \int_{X_\alpha} d_{\mathrm{dt}}\Big((m_{\beta,\alpha}\circ f_\beta\circ l_{\alpha,\beta})(x),
f_\alpha(x)\Big)\ d\mu_\alpha\\
&+\int_{X_\alpha} d_{\mathrm{dt}}\Big((m_{\beta,\alpha}\circ m_{\gamma,\beta}\circ f_\gamma\circ l_{\beta,\gamma}\circ l_{\alpha,\beta})(x), (m_{\beta,\alpha}\circ f_\beta \circ l_{\alpha,\beta})(x)\Big)\ d\mu_\alpha\\
\ge& \int_{X_\alpha} d_{\mathrm{dt}}\Big((m_{\beta,\alpha}\circ m_{\gamma,\beta}\circ f_\gamma\circ l_{\beta,\gamma}\circ l_{\alpha,\beta})(x),f_\alpha(x)\Big)\ d\mu_\alpha=\cost(\pi_2\circ \pi_1)
\end{array}\]
where the second to last inequality follows from the triangle inequality for $d_{\mathrm{dt}}$.
%
Therefore, $\cost(\pi_1)+\cost(\pi_2)\ge\cost(\pi_2\circ \pi_1)$.
It follows that
\begin{align*}
&\inf \{\cost(\pi') \mid \pi'\colon \alpha \leftrightharpoons_{\mathbf{T}} \beta\}+ \inf \{\cost(\pi'') \mid \pi'' \colon \beta\leftrightharpoons_{\mathbf{T}} \gamma\}\\
&=\inf\{\cost(\pi')+\cost(\pi'') \mid \pi'\colon \alpha \leftrightharpoons_{\mathbf{T}} \beta, \pi'' \colon \beta\leftrightharpoons_{\mathbf{T}} \gamma \}\\
&\ge\inf \{\cost(\pi''\circ \pi') \mid \pi'\colon \alpha \leftrightharpoons_{\mathbf{T}} \beta,\pi''\colon \beta \leftrightharpoons_{\mathbf{T}} \gamma\}\\
&\ge \inf\{\cost(\pi) \mid \pi\colon \alpha \leftrightharpoons_{\mathbf{T}} \gamma\}
\end{align*}
and thus 
$h_{\mb{O}}(\alpha,\beta)+h_{\mb{O}}(\beta,\gamma)
\ge h_{\mb{O}}(\alpha,\gamma)$.
In other words, (T) holds.

To prove (R) i.e., that for all GEOs $(f_\alpha,t_\alpha)\colon (X_\alpha,G_\alpha) \to (Y_\alpha, K_\alpha)$, it holds that  $h_{\mb{O}}\Big((f_\alpha,t_\alpha),(f_\alpha,t_\alpha)\Big)=0$, observe that, since ${\mathbf{T}}$ is a category there exists the crossed pair of translation 
$\iota\defeq \Big( (\id_{X_\alpha},\id_{G_\alpha}), (\id_{Y_\alpha},\id_{K_\alpha})\Big)$
given by the identity morphisms. One can easily check that $\cost(\iota)=0$ and thus \[\inf \{\cost(\pi) \mid \pi\colon (f_\alpha,t_\alpha)\leftrightharpoons_{\mathbf{T}}(f_\alpha,t_\alpha)\}=0\text{.}\]

\paragraph{Proof of Proposition \ref{th:upper_bound}.}
Fix $A\defeq\{(g,x)\mid f_{\alpha}(x)=f_\beta(x)\}$, 
$B\defeq\{(g,x)\mid f_{\alpha}(g*x)= f_\beta(g*x)\}$ and 
$C\defeq\{(g,x)\mid f_\beta(x)= f_\beta(g*x)\}$ and observe that $A\cap B \subseteq C$. Thus, by denoting with $\overline{X}$, the complement of a set $X$, it holds that $\overline{A} \cup \overline{B}\supseteq \overline{C}$ and thus
\begin{equation}\label{eq:proof theorem 2}
|\overline{A}|+|\overline{B}| \geq |\overline{C}|\text{.}
\end{equation}
We now use the hypothesis that $G_\alpha$ is a group, to show the bijection of $\overline{A}$ and $\overline{B}$: define $\iota\colon\overline{B}\to \overline{A}$ as $\iota(g,x)\defeq (g,g*x)$ and $\kappa\colon \overline{A} \to \overline{B}$ as $\kappa(g,x)\defeq (g,g^{-1}*x)$. Observe that the functions are well defined and that they are inverse to  each other. Thus $|\overline{A}|=|\overline{B}|$ that, thanks to \eqref{eq:proof theorem 2} gives us 
\[2\cdot|\overline{A}| \geq |\overline{C}|\text{.}\]
To conclude observe that $\overline{C}$ is $NE$ and that $|\overline{A}|$ is $|G_\alpha|\cdot h_{\mathbb{O}}((f_\alpha,t_\alpha),(f_\beta,t_\beta))$.

\end{document}